\documentclass[acmsmall,screen]{acmart}

\usepackage{amsmath}
\usepackage{amsthm}
\usepackage{algorithm}
\usepackage{algorithmic}
\usepackage{xspace}
\usepackage{graphicx}
\usepackage{subfigure}
\usepackage{wrapfig}
\usepackage{graphicx}
\usepackage{url}
\usepackage{color}
\usepackage{multirow}
\usepackage{gensymb}
\usepackage{longtable}
\usepackage{tikz}
\usepackage{comment}
\usepackage{dsfont}
\usepackage{framed}
\usepackage{enumitem}
\usepackage{xspace}
\usepackage{bm}

\newcommand{\ouralg}{EC-L2O\xspace}
\newcommand{\ouralgcal}{MLA-ROBD\xspace}
\newcommand{\ouralgpro}{EC-L2O\xspace}

\newcommand{\mx}[1]{\begin{bmatrix} #1 \end{bmatrix}}
\newcommand{\revise}[1]{{{#1}}}

\theoremstyle{plain}
\newtheorem{theorem}{Theorem}[section]

\newtheorem{lemma}[theorem]{Lemma}

\theoremstyle{definition}
\newtheorem{definition}{Definition}

\setcopyright{rightsretained}
\acmJournal{POMACS}
\acmYear{2022} \acmVolume{6} \acmNumber{2} \acmArticle{28} \acmMonth{6} \acmPrice{}\acmDOI{10.1145/3530894}


\received{February 2022}
\received[revised]{April 2022}
\received[accepted]{April 2022}

\begin{document}

\title{Expert-Calibrated Learning for Online Optimization with Switching Costs}\thanks{*Pengfei Li and Jianyi Yang contributed equally. This work was supported in part by the U.S. NSF under grants CNS-1551661 and CNS-2007115.}

\author{Pengfei Li}
\email{pli081@ucr.edu}
\affiliation{%
  \institution{University of California, Riverside}
  \streetaddress{900 University Ave.}
  \city{Riverside}
  \state{California}
  \postcode{92521}
  \country{United States}
}

\author{Jianyi Yang}
\email{jyang239@ucr.edu}
\affiliation{%
	\institution{University of California, Riverside}
	\streetaddress{900 University Ave.}
	\city{Riverside}
	\state{California}
	\postcode{92521}
    \country{United States}
}

\author{Shaolei Ren}
\email{sren@ece.ucr.edu}
\affiliation{%
	\institution{University of California, Riverside}
	\streetaddress{900 University Ave.}
	\city{Riverside}
	\state{California}
	\postcode{92521}
    \country{United States}
}

\renewcommand{\shortauthors}{Pengfei Li et al.}

\begin{abstract}
We study online convex optimization with switching costs, a practically important but also extremely challenging problem due to the lack of complete offline information.
By tapping into
the power of machine learning (ML) based optimizers, ML-augmented online algorithms
(also referred to as expert calibration in this paper)
have been emerging as state of the art, with provable worst-case performance guarantees.
Nonetheless,
by using the standard
practice of training an ML model
as a standalone optimizer
and plugging it
into an ML-augmented algorithm, the average cost performance can be highly unsatisfactory.
In order to address the ``\emph{how to learn}''  challenge, we propose
\ouralg (expert-calibrated learning to optimize), which
trains
an ML-based optimizer by explicitly taking into account the downstream expert calibrator.
To accomplish this, we propose a new differentiable
expert calibrator that generalizes regularized online balanced descent and
offers a provably better
competitive ratio than pure ML predictions when the prediction error is large.
For training, our loss function is a weighted sum of two different losses --- one
minimizing the average ML prediction error for better robustness,
 and the other one minimizing the post-calibration
 average cost.
 We also provide theoretical analysis for \ouralg, highlighting
that expert calibration can be even beneficial for the average cost performance
and that the
high-percentile tail ratio of the cost achieved by \ouralg to that of the offline optimal oracle
(i.e., tail cost ratio) can be bounded.
Finally, we test \ouralg by running simulations for  sustainable datacenter demand response.
Our results demonstrate that \ouralg can empirically achieve a lower average cost as well
as a lower competitive ratio than the existing baseline algorithms.
\end{abstract}

\maketitle

\section{Introduction}

Many real-world problems,
such as energy scheduling in smart grids,
resource management in clouds and rate adaptation in video streaming \cite{SOCO_DynamicRightSizing_Adam_Infocom_2011_LinWiermanAndrewThereska,SOCO_Memory_Adam_NIPS_2020_NEURIPS2020_ed46558a,SOCO_OBD_Niangjun_Adam_COLT_2018_DBLP:conf/colt/ChenGW18,SOCO_OnlineMetric_UntrustedPrediction_ICML_2020_DBLP:conf/icml/AntoniadisCE0S20,SOCO_Prediction_Error_RHIG_NaLi_Harvard_NIPS_2020_NEURIPS2020_a6e4f250}, can be formulated
as online convex optimization where an
agent makes actions
online based on
sequentially revealed information. The goal of the agent is
to
minimize the sum of convex costs over an episode
of multiple time steps. In addition, another
crucial concern is that changing
actions too abruptly is highly undesired in most practical applications. For example, frequently turning
on and off servers in data centers can decrease
the server lifespan and even create dangerous
situations such as
high inrush current  \cite{Oversubscription_BatteryCoordinated_Qingyuan_Facebook_MICRO_2020_mallacoordinated2020},
and a robot cannot arbitrarily change
its position due to velocity constraints in navigation tasks \cite{Robot_VelocityConstraint_IEEE_2008_4543969}.
Consequently, such action impedance has motivated
an emerging
area of online convex optimization with switching
costs, where the switching cost measures
the degree of action changes across two consecutive time steps and acts a regularizer
 to make the online actions smoother.

While both empirically and theoretically important, online convex optimization with switching costs is extremely challenging. The key reason is that
the optimal actions at different time steps are highly dependent on each other and hence require the complete offline future information, which is nonetheless lacking in practice \cite{SOCO_Prediction_Error_Niangjun_Sigmetrics_2016_10.1145/2964791.2901464,SOCO_CapacityScalingAdaptiveBalancedGradient_Gatech_MAMA_2021_Sigmetrics_2021_10.1145/3512798.3512808}.

To address the lack of offline future information
(i.e., future cost functions or  parameters),
the set of algorithms that make actions
only using online information have been quickly
expanding. For example, some prior studies have
considered
online gradient descent (OGD) \cite{OGD_zinkevich2003online,SOCO_Prediction_Error_Meta_ZhenhuaLiu_SIGMETRICS_2019_10.1145/3322205.3311087}, online balanced descent (OBD) \cite{SOCO_OBD_Niangjun_Adam_COLT_2018_DBLP:conf/colt/ChenGW18}, and regularized OBD (R-OBD) \cite{SOCO_OBD_R-OBD_Goel_Adam_NIPS_2019_NEURIPS2019_9f36407e}.
Additionally, some other studies have also incorporated machine learning (ML) prediction
of future cost parameters  into the algorithm design under various settings.
Notable examples include
 receding horizon control (RHC) \cite{SOCO_Prediction_Error_Meta_ZhenhuaLiu_SIGMETRICS_2019_10.1145/3322205.3311087}
committed horizon control (CHC) \cite{SOCO_Prediction_Error_Niangjun_Sigmetrics_2016_10.1145/2964791.2901464},
receding horizon gradient descent (RHGD) \cite{Receding_Horizon_GD_li2020online}
and
adaptive balanced capacity scaling (ABCS) \cite{SOCO_CapacityScalingAdaptiveBalancedGradient_Gatech_MAMA_2021_Sigmetrics_2021_10.1145/3512798.3512808}.
Typically, these algorithms are developed based on classic optimization frameworks and offer guaranteed performance robustness in terms of the competitive ratio, which measures the
worst-case ratio of the cost achieved by an online algorithm to the offline oracle's cost.
Nonetheless, despite the theoretical guarantee, a bounded worst-case competitive ratio
does not necessarily translate into decreased average cost in most typical cases.

By exploiting the abundant historical data available
in many practical applications (e.g., server management data center
and energy scheduling in smart grids), the power of ML can go much further beyond  simply predicting the future cost parameters. Indeed,  state-of-the-art \emph{learning-to-optimize} (L2O) techniques
can even substitute conventional optimizers
and directly predict online actions \cite{L2O_Combinatorial_Optimization_Survey_Yoshua_2021_BENGIO2021405,L2O_LearnWithoutGraident_ICML_2017_l2o_gradeint_descent_chen_ICML,L2O_NewDog_OldTrick_Google_ICLR_2019}. Thus, although it still remains
under-explored in the context of online convex optimization with switching costs,
the idea of using statistical ML to discover the otherwise hidden
mapping from online information
to actions is very natural and promising.

While
ML-based optimizers can often result in a low average cost in typical cases, its drawback is also significant --- lack of performance robustness.
Crucially, albeit rare, some input instances can be arbitrarily ``bad'' for
a pre-trained ML-based optimizer and empirically
result in a high competitive ratio. This
is a fundamental limitation of any statistical ML models, and can be attributed to several factors, including certain testing inputs drawn from a very different distribution than the training distribution, inadequate ML model capacity, ill-trained ML model weights, among others.

To provide worst-case performance guarantees and potentially leverage the advantage of an ML-based optimizer,  ML-augmented algorithm designs have been recently considered in the context of online convex optimization with switching costs \cite{SOCO_OnlineMetric_UntrustedPrediction_ICML_2020_DBLP:conf/icml/AntoniadisCE0S20,SOCO_OnlineOpt_UntrustedPredictions_Switching_Adam_arXiv_2022,SOCO_ML_ChasingConvexBodiesFunction_Adam_UnderSubmission_2022}.
These algorithms often take a simplified view
of the ML-based optimizer ---
the actions can come from any exogenous source,
and an ML-based optimizer is just one of the possible sources.
  Concretely, the exogenous actions
  (viewed as ML predictions) are
fed into another algorithm and modified
following an expert-designed rule.
We refer to this process
as expert \emph{calibration}.
 Thus, instead of using the original un-calibrated ML predictions,
the agent adopts new expert-calibrated actions.  Such expert calibration is inevitably crucial to achieve otherwise impossible performance robustness
in terms of the competitive ratio, thus addressing
the fundamental limitation of ML-based optimizers.
But, unfortunately, a loose (albeit finite) upper bound on the competitive ratio in and of itself
does not always lead to a satisfactory average cost performance.
A key reason is that the standard practice
is to train the ML-optimizer as a standalone optimizer
that can give good actions
 on its own in many typical cases, but the
 the already-good actions in these cases will still be subsequently altered by expert calibration
 and hence lead to an increased average cost.

In this paper, we study online convex optimization
with switching costs and
propose \ouralg (Expert-Calibrated Learning to Optimize), a novel expert-calibrated ML-based optimizer that minimizes the average cost while provably
improving performance robustness compared to pure ML predictions.
Concretely, based on a recurrent structure illustrated in Fig.~\ref{fig:illustration},
\ouralg integrates a \emph{differentiable} expert calibrator  as a new downstream implicit layer following
the ML-based optimizer.
\revise{The expert calibrator generalizes
state-of-the-art R-OBD \cite{SOCO_OBD_R-OBD_Goel_Adam_NIPS_2019_NEURIPS2019_9f36407e}
by adding a regularizer to keep
R-OBD's actions close to ML predictions.}
 At each time step,
the ML-based optimizer takes online information as its input and predicts an action,
which is then calibrated by the expert algorithm.
Importantly,
we
view the combination of the ML-based optimizer
and the expert calibrator as an integrated entity,
and holistically train the ML model to minimize the average \emph{post}-calibration cost.

Compared
to the standard practice of training
a standalone ML model independently to minimize
the average \emph{pre}-calibration cost,
the ML-based optimizer in \ouralg is fully
aware of the downstream expert calibrator,
thus effectively mitigating the undesired
average cost increase induced by expert calibration.
Moreover, the inclusion of a differentiable expert calibrator allows us to leverage backpropagation
to efficiently train the ML model, and also substantially improves
the performance robustness compared to a standalone ML-based optimizer.

We rigorously prove that when the ML prediction errors are large, expert-calibrated
ML predictions are guaranteed to improve the competitive ratio compared to
pre-calibration predictions.
\revise{On the other hand, when the predictions are of a good quality,
the expert-calibrated ML predictions can lower the
optimal competitive ratio achieved by standard R-OBD without predictions. Interestingly, with proper settings,
the expert-calibrated ML predictins can achieve a sublinear cost regret
compared to the ($L$-constrained) optimal oracle's actions.}
 We also provide the average cost bound,
highlighting that expert calibration can benefit ML predictions when the training-testing distribution discrepancy is large. Additionally, we
provide a bound on the
high-percentile tail ratio of the cost achieved by \ouralg to that of the offline optimal oracle
(i.e., tail cost ratio). Our analysis demonstrates that \ouralg
can achieve both a low average cost and a bounded tail cost ratio
by training the ML-based optimizer
in \ouralg over a weighted sum of two different losses --- one for the pre-calibration
ML predictions to have
a low average prediction error, and the other one for the post-calibration

Finally, we test \ouralg by running simulations for the application of sustainable datacenter demand response. Under a practical setting
with training-testing distributional discrepancy,
our results demonstrate that \ouralg can achieve a lower average cost than expert algorithms
as well as the pure ML-based optimizer.
More interestingly, \ouralg also achieves an empirically lower competitive ratio  than the existing expert algorithms, resulting in the best average cost vs. competitive ratio tradeoff.
This nontrivial result is due in great part to the
differentiable expert calibrator and the
loss function that we choose --- both the average cost and ML prediction errors
(which affect the competitive ratio) are taken into account in the training process.

To sum up, \ouralg is the first to address the ``\emph{how to learn}'' challenge
for online convex
 optimization with switching costs.
 By integrating a differentiable expert calibrator and training
 the ML-based optimizer on a carefully designed loss function,
 \ouralg minimizes the average cost, provably improves the performance robustness
 (compared to pure ML predictions), and offers probabilistic guarantee on the high-percentile tail cost ratio.

\section{Problem Formulation}

In general online optimization problems \cite{tutorial_online_learning_orabona2019modern}, actions are made online according to sequentially revealed information to minimize the sum of the costs for an episode with a length of $T$
steps. Moreover, smooth actions are desired in many practical
problems, such as dynamic server on/off provisioning
in data centers and robot movement.
Thus,
we consider that changing actions across two consecutive steps also induces an additional cost referred to as \emph{switching}
cost, which acts as a regularizer
and tends to make the online actions
smoother over time.

More concretely, at the beginning
of each step $t=1,\cdots, T$, the online information (a.k.a. context) $y_t\in\mathcal{Y}_t\subseteq \mathcal{R}^q$ is revealed to the agent, which
then decides an action $x_t$ from an action set $\mathcal{X}\subseteq\mathcal{R}^d$.
 With action $x_t$, the agent incurs a hitting (or operational) cost $f(x_t,y_t)$, which is parameterized by the context $y_t$ and convex with respect to $x_t$,
plus an additional switching cost $c(x_t,x_{t-1})$ measuring how large the action change is.
The goal of the agent is to minimize the sum of the hitting  costs and the switching costs over an episode of $T$ steps as follows:
\begin{equation}\label{eqn:obj}
	\min_{x_1,\cdots x_T} \sum_{t=1}^Tf(x_t,y_t)+c(x_t,x_{t-1}),
\end{equation}
where the initial action $x_0$ is provided
 as an additional input to the agent prior
to its first online action.
The complete offline information is thus denoted as $\bm{s}=\left(x_0,\bm{y} \right)\in\mathcal{S}$ with
 $\bm{y}=\left[y_1,\cdots,y_T \right]$
 and $\mathcal{S}=\mathcal{X}\times\prod_{t=1}^T \mathcal{Y}_t$.
 While we can
 formulate
it as an equivalent hard constraint,
our inclusion of the switching
cost as a soft regularizer
 is well consistent with the existing literature \cite{SOCO_Memory_Adam_NIPS_2020_NEURIPS2020_ed46558a,SOCO_NonConvex_Adam_Sigmetrics_2020_10.1145/3379484,SOCO_OBD_Niangjun_Adam_COLT_2018_DBLP:conf/colt/ChenGW18}.

The key challenge of solving
Eqn.~\eqref{eqn:obj} comes
from the action entanglement due to the switching
cost: the complete contextual information
$\bm{s}=\left(x_0,\bm{y} \right)$ is required
to make optimal actions, but it is lacking in advance and only sequentially revealed to the agent online.

\textbf{Assumptions.}
We assume that the hitting cost function $f(x_t,y_t)$ is non-negative and $m$-strongly convex in $x_t$, which has also been
considered in prior studies \cite{SOCO_OBD_LQR_Abstract_Goel_Adam_Caltech_2019_10.1145/3374888.3374892,SOCO_OBD_R-OBD_Goel_Adam_NIPS_2019_NEURIPS2019_9f36407e}.
In addition, the
 switching cost $c(x_t,x_{t-1})$ is measured
 in terms of the squared Mahalanobis distance with respect to a symmetric and positive-definite  matrix $Q\in\mathcal{R}^{d\times d}$, i.e. $c(x_t,x_{t-1})=\left(x_t-x_{t-1}\right) ^\top Q \left(  x_t-x_{t-1}\right) $  \cite{mahalanobis_distance_de2000mahalanobis}.  Here, we assume that the smallest eigenvalue of $Q$ is $\frac{\alpha}{2}>0$ and the largest eigenvalue of $Q$ is $\frac{\beta}{2}$, which means $\frac{\alpha}{2}\|x_t-x_{t-1}\|^2\leq c(x_t,x_{t-1})\leq \frac{\beta}{2}\|x_t-x_{t-1}\|^2$.
 The interpretation is that
  the switching cost quantifies
  the impedance of action movements in a linearly
  transformed space. For example,
 a diagonal matrix
 $Q$ with different non-negative diagonal elements
 can model the practical consideration
 that
   the change of a multi-dimensional action along certain
   dimensions can incur a larger cost than others
   (e.g., it might be easier
   for a flying drone to move horizontally than vertically).
 In the special case of $Q$ being an identity matrix, the switching cost reduces to the quadratic cost considered in \cite{SOCO_OBD_LQR_Abstract_Goel_Adam_Caltech_2019_10.1145/3374888.3374892}.

\textbf{Performance metrics.}
For an online algorithm $\pi$, we denote its total
cost, including hitting and switching costs,
for a problem instance with context $\bm{s}=(x_0,\bm{y})$
 as $\mathrm{cost}(\pi,\bm{s})=\sum_{t=1}^Tf(x_t,y_t)+c(x_t,x_{t-1})$ where $x_t$, $t=1,\cdots, T$,
 are the actions produced by the algorithm $\pi$.
Likewise, by denoting $\pi^*$ as
the offline optimal oracle that has access
to
the complete information $\bm{s}=(x_0,\bm{y})$
in advance and selects actions $x_t^*$, $t=1,\cdots, T$, we write the offline optimal cost
as $\mathrm{cost}(\pi^*,\bm{s})=\sum_{t=1}^Tf(x^*_t,y_t)+c(x^*_t,x^*_{t-1})$  with $x^*_{0}=x_0$.

The contextual information
$\bm{s}=(x_0,\bm{y})\in\mathcal{S}$ is drawn from an exogenous joint distribution $\mathbb{P}$.
To evaluate the performance of
$\pi$, we focus on two most important
metrics --- average cost and competitive ratio ---
defined as follows.

\begin{definition}[Average cost]\label{definition:average_cost}
For contextual information
$\bm{s}=(x_0,\bm{y})\sim\mathbb{P}$, the average
cost of an algorithm $\pi$ over the joint distribution
$\mathbb{P}$ is defined as
\begin{equation}
	\mathrm{AVG}(\pi)=\mathbb{E} \left[\mathrm{cost}(\pi,\bm{s}) \right].
\end{equation}
\end{definition}

\begin{definition}[Competitive ratio]\label{definition:average_cr}
The competitive ratio of an algorithm $\pi$
is defined as
\begin{equation}\label{eqn:cr}
	\mathrm{CR}(\pi)=\sup_{\bm{s}\in\mathcal{S}}\frac{\mathrm{cost}(\pi,\bm{s})}{\mathrm{cost}(\pi^*,\bm{s})}.
	\end{equation}
\end{definition}

The average cost measures the expected performance
in typical cases that an algorithm can attain given an environment distribution. On the other hand, the competitive
ratio measures the worst-case performance of an algorithm relative
to the offline optimal cost
for any feasible problem instance that might
be presented by the environment. While the average cost is important in practice, the conservative metric of competitive ratio quantifies the level of robustness of an algorithm. Importantly,
the two metrics are
different from each other: a lower average cost does not necessarily mean a lower competitive ratio, and vice versa.

Similar to the worst-case competitive ratio,
we also consider the high-percentile tail ratio of the cost achieved by an online algorithm to that of the offline optimal oracle (Section~\ref{sec:tail_cost_ratio}).
This metric, simply referred to as the tail cost ratio, provides a probabilistic view of the performance robustness of an algorithm.

\section{A Simple ML-Based Optimizer}

As a warm-up,  we present 
a simple
 ML-based optimizer that is trained
as a standalone model to produce
good actions on its own. We
emphasize that, while the ML-based
optimizer can result in a low average cost,  it has significant limitations
in terms of the worst-case performance robustness.

\subsection{Learning a Standalone Optimizer}\label{sec:pure_ML}

Due to the agent's lack of complete contextual information
$\bm{s}=(x_0,\bm{y})$
in advance, the crux of
online convex optimization with switching costs
is how to map the online information
to an appropriate action so as to minimize the cost. Fortunately, most practical applications (e.g., server management data center
and energy scheduling in smart grids) have abundant historical data available. This can be exploited
by state-of-the-art ML techniques (e.g.,
L2O \cite{L2O_Combinatorial_Optimization_Survey_Yoshua_2021_BENGIO2021405,L2O_LearnWithoutGraident_ICML_2017_l2o_gradeint_descent_chen_ICML})
and hence lead to the natural but still under-explored idea of using statistical ML to discover the otherwise difficult
mapping from online contextual information
to actions.

More concretely, following the state-of-the-art L2O techniques,
we can first train a standalone ML-based optimizer $h_{W}$ parameterized by the model weight $W$, over a training dataset of historical and/or
synthetic problem instances. For example,
because of the universal approximation capability, we can employ a deep neural network (DNN) with recurrent structures
as the underlying ML model $h_W$, where the recurrence accounts for the online optimization process.
With online contextual
information $y_t$ and the previous action $\tilde{x}_{t-1}$ as input,
the ML-based optimizer with output $\tilde{x}_t$ can be trained
to imitate the action $x^{\pi}_t$ of an expert (online) algorithm $\pi$ with good average or worst-case cost performance. Towards this end, the imitation loss
defined in terms of the distance between
the expert action and learnt action can be used
as the loss function in the training process.
 Alternatively,   we can also directly
use the total cost over an entire episode  as the loss function
to supervise offline training of the ML-based optimizer \cite{L2O_Combinatorial_Optimization_Survey_Yoshua_2021_BENGIO2021405}.

Given an unseen testing problem instance for
 inference,
at each step,   the pre-trained ML-based optimizer takes online contextual
information and the previous action as its input, and predicts the current action
as its output. Without causing ambiguity, we simply
use ML \emph{predictions} to represent
the actions produced by the ML-based optimizer.

\subsection{Limitations}\label{sec:ML_limitation}

With a large training dataset, the standalone ML-based optimizer can
exploit the power of statistical learning and
achieve a low cost on average, especially when the testing
input distribution is well consistent with the training distribution (i.e., in-distribution testing).
But, regardless of in-distribution or out-of-distribution testing, an ML-based optimizer can perform very poorly for some ``bad'' problem instances and empirically result in a high competitive ratio. This is a fundamental limitation of any statistical ML models, be they advanced DNNs or simple linear models.

The lack of performance robustness in these rare cases
can be attributed
to several factors, including
certain testing inputs drawn from a very different
distribution than the training distribution,
inadequate ML model capacity, ill-trained ML
model weights, among others.
For example, even though the ML-based optimizer
can perfectly imitate an expert's action or obtain optimal actions for
all the \emph{training} problem instances, there is no
guarantee that it can perform equally well for
all \emph{testing} problem instances
due to the ML model capacity, inevitable generalization error as well as other factors.

While distributionally robust learning can alleviate the problem of testing distributional shifts relative to the training distribution
\cite{Bandit_RobustDistributionalContextual_AISTATS_2020_kirschner2020distributionally,DRO_LabelShiftDRO_ICLR_2021_zhang2021coping}, it
still targets the average case (albeit over a wider distribution of testing problem instances) without addressing the worst-case performance, let alone
its significantly increased training complexity.

To further quantify the lack of performance robustness, we define $\rho$-accurate prediction for the ML-based optimizer $h_W$ as follows.

\begin{definition}[$\rho$-accurate Prediction]\label{def:predacc}
Assume that for an input  instance $\bm{s}=(x_0,\bm{y})$, the offline-optimal oracle $\pi^*$ gives $x_1^*,\cdots x_T^*$, and the offline optimal cost is 
$\mathrm{cost}(\pi^*,\bm{s})=\sum_{t=1}^Tf(x^*_t,y_t)+c(x^*_t,x^*_{t-1})$. The predicted actions from the ML-based optimizer $h_W$, denoted as $\tilde{x}_1,\cdots \tilde{x}_T$, are said to be $\rho$-accurate when the following is satisfied:
\begin{equation}\label{eqn:rho_accurate}
\sum_{t=1}^T ||\tilde{x}_t-x_t^* ||^2 \leq \rho\cdot \mathrm{cost}(\pi^*,\bm{s}),
\end{equation}
where $\rho\geq0$ is referred to as the prediction error, and $||\cdot||$ is the $l_2$-norm.
\end{definition}

The prediction error $\rho$
is defined similarly as in the existing literature \cite{SOCO_OnlineMetric_UntrustedPrediction_ICML_2020_DBLP:conf/icml/AntoniadisCE0S20,SOCO_OnlineOpt_UntrustedPredictions_Switching_Adam_arXiv_2022}
and normalized with respect to the offline optimal oracle's total cost. It is also scale-invariant --- when both the sides of Eqn.~\eqref{eqn:rho_accurate} are scaled by the same factor, the prediction error remains unchanged.
 By definition,
 a smaller $\rho$ implies a better prediction,
 and
 the ML predictions are perfect
 when $\rho = 0$.

In the following lemma, given $\rho$-accurate predictions, we provide a lower bound on the
ratio of the cost achieved by the ML-based optimizer $h_W$
to that of the offline oracle. Note
that, with a slight abuse
of notion, we also refer to this ratio
as the \emph{competitive ratio} denoted by
$CR_{\rho}(h_W)=\sup_{\bm{s}\in\{\bm{s}\in\mathcal{S}| \tilde{x}_1,\cdots \tilde{x}_T \text{ are }
\rho\text{-}\text{accurate}\}}\frac{\mathrm{cost}(h_W,\bm{s})}{\mathrm{cost}(\pi^*,\bm{s})}$, where the subscript $\rho\geq0$
emphasizes that the competitive ratio
is restricted over all the
instances whose corresponding ML predictions
are $\rho$-accurate.

\begin{lemma}[Competitive ratio lower bound
for $\rho$-accurate predictions]\label{thm:ml_eta}
Assume that the hitting cost function $f$ is $m$-strongly convex in terms of the action
and $\alpha$ is twice the smallest
 eigenvalue of the matrix $Q$ in the switching cost (i.e., $\frac{\alpha}{2}$ is the smallest
 eigenvalue of $Q$).
For all $\rho$-accurate predictions, the competitive ratio of
the ML-based optimizer $h_W$ satisfies
$CR_{\rho}(h_W)\geq 1 + \frac{m + 2\alpha}{2} \rho$.
\end{lemma}

Lemma~\ref{thm:ml_eta} is proved in Appendix~\ref{sec:proof_ml_eta}.
It provides a competitive ratio
lower bound for an ML-based optimizer with $\rho$-accurate predictions, showing
that the competitive ratio grows at least linearly
with respect to the ML prediction error $\rho$.
The slope of increase depends on
convexity parameter $m$ of the hitting cost function
and the smallest
 eigenvalue of the matrix $Q$ in the switching cost. Nonetheless, the actual competitive ratio can be arbitrarily bad for $\rho$-accurate predictions. For example, if the hitting cost function is non-smooth
 with an unbounded  second-order derivative,
 then the competitive ratio can also be unbounded.
Thus, even when the prediction error $\rho$
is empirically bounded in practice, the actual competitive ratio of the pure ML-based optimizer $h_W$
can still be arbitrarily high,
highlighting the lack of performance robustness.

\section{\ouralg: Expert-Calibrated ML-Based Optimizer}

In this section, we first show that simply combining
a standalone ML-based optimizer with an expert algorithm can even
result in an increased average cost compared to using the ML-based optimizer alone. Then,
we present \ouralg, a novel expert-calibrated
ML-based optimizer  that minimizes the average cost while improving performance robustness.

\subsection{A Two-Stage Approach and Drawbacks}\label{sec:calibration_limitation_simple}

By using L2O techniques \cite{L2O_Combinatorial_Optimization_Survey_Yoshua_2021_BENGIO2021405,L2O_LearningToOptimize_Berkeley_ICLR_2017,L2O_LearnWithoutGraident_ICML_2017_l2o_gradeint_descent_chen_ICML}  (discussed in Section~\ref{sec:pure_ML}),
a standalone ML-based optimizer
$h_W$
can predict online actions with a low
average cost,
but lacks performance robustness in terms of the competitive ratio.
On the other hand, expert algorithms can take
the ML predictions as input and then produce
new calibrated online actions with good competitive ratios.
This is also referred to as ML-augmented algorithm design \cite{SOCO_OnlineMetric_UntrustedPrediction_ICML_2020_DBLP:conf/icml/AntoniadisCE0S20,SOCO_OnlineOpt_UntrustedPredictions_Switching_Adam_arXiv_2022}.

Consequently, to perform well both in average cases
and in the worst case, it seems very straightforward to
combine states of the art from both L2O and ML-augmented algorithm design
using a \emph{two}-stage approach:
\emph{first}, we train an ML-based optimizer $h_W$
that has good
average performance on its own;
\emph{second}, we add
another expert-designed algorithm, denoted by $R$, to calibrate the ML predictions for performance robustness.
 By doing so, we have a new expert-calibrated
optimizer $R\oplus h_W$, where $\oplus$ denotes the simple composition of
the expert calibrator $R$ with
the independently trained ML-optimizer $h_W$.
The new optimizer
$R\oplus h_W$ still takes the same online input
information as the original ML-based optimizer $h_W$, but produces calibrated actions on top of the original ML predictions.

While the expert-calibrated optimizer $R\oplus h_W$ reduces the
worst-case competitive ratio than pure ML predictions,
its average cost can still be high.
\revise{The key reason can be attributed to
 the ML-based optimizer's obliviousness
of the downstream expert calibrator during its training process.
Specifically,
 the ML-based optimizer $h_W$
is trained as a \emph{standalone} model to have
good average performance on its own, but the
 good ML predictions
 for many typical input instances
 are  subsequently modified
by the expert algorithm (i.e., calibration)
that may not perform well on average.
 We will further explain this point by our performance analysis in Section~\ref{sec:analysis_average}.
}

In summary, despite the improved worst-case robustness,
the new optimizer $R\oplus h$ constructed using a naive two-stage approach
can have an unsatisfactory average cost performance.

\subsection{Overview of \ouralg}

To address the average cost drawback
of the simple two-stage approach, we propose \ouralg, which trains the ML-based optimizer $h_W$ by explicitly
 taking into
account the downstream expert calibrator.
The high-level intuition is that if only expert-calibrated ML predictions are used
as the agent's actions (for performance robustness), it would naturally make more sense
to train the ML-based optimizer in such a way that
the post-calibration cost is minimized.

To distinguish
from the naive calibrated optimizer $R\oplus h_W$,
we use $R\circ h_W$ to
denote our expert-calibrated optimizer, where
$\circ$ denotes the function composition.
Next,
we present the overview of inference/testing and training
of \ouralg, which is also illustrated in Fig.~\ref{fig:illustration}.

\begin{figure*}[!t]	
\centering
	\includegraphics[width=0.8\textwidth]{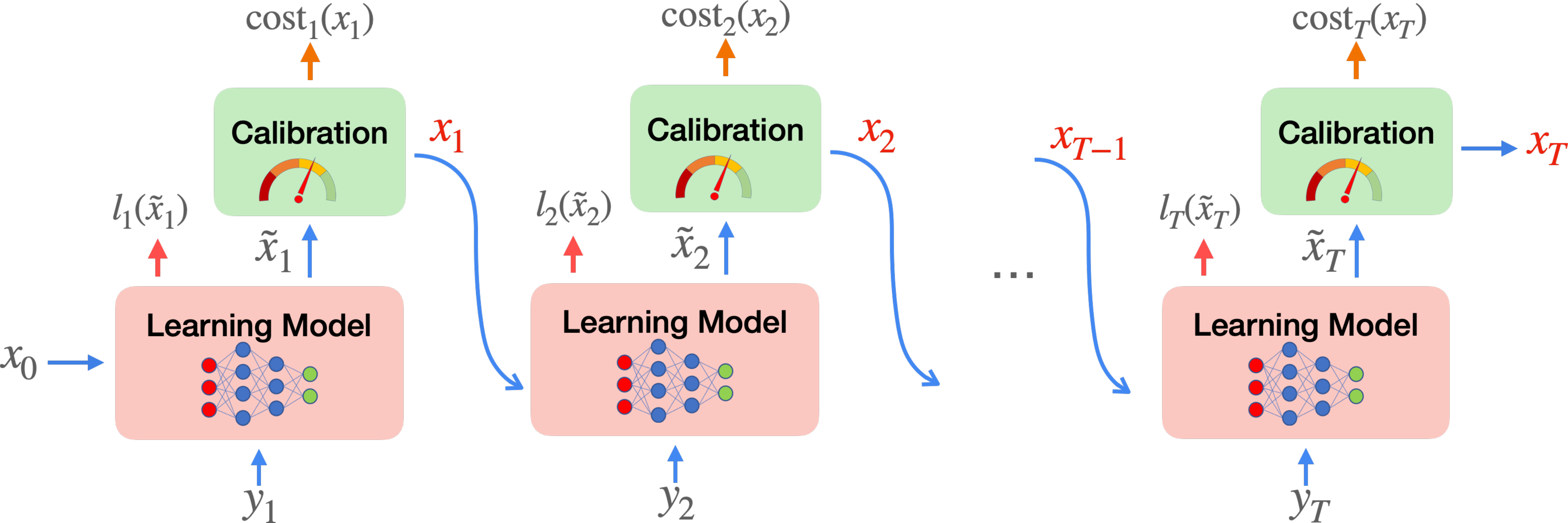}
	\vspace{-0.3cm}
	\caption{Illustration of \ouralgpro. The ML-based optimizer $h_W$ is trained to minimize a weighted sum of two losses in view of the downstream expert calibrator $R_{\lambda}$. The expert calibrated actions $x_1,\cdots,x_T$ are actually used by the online agent and directly determines the cost.}\label{fig:illustration}
\end{figure*}

\subsubsection{Online testing} At time step $t$,
\ouralg
takes the available online contextual information $y_t$ as well as the previous action $x_{t-1}$,
and outputs the current action $x_t$. More specifically, given $y_t$ and $x_{t-1}$,
the ML-based optimizer $h$ predicts the action $\tilde{x}_t=h_W(y_t, x_{t-1})$, which is then calibrated as the actual action $x_t$ by
the expert algorithm $R$.
Due to the presence of the expert calibrator $R$, the optimizer  $R\circ h_W$
is more robust than simply using the ML predictions alone.
This process follows a
recurrent structure as illustrated in Fig.~\ref{fig:illustration}, where
the base optimizers $R\circ h_W$ are concatenated one
after another.

\subsubsection{Offline training} The ML-based optimizer $h_W$ in \ouralg is constructed using an ML model (e.g.,
neural network)
with trainable weights $W$.
But, unlike in $R\oplus h_W$ constructed using the naive two-stage approach in which the ML-based optimizer $h_W$
is trained as a standalone model to produce good online actions on its own,
we train $h_W$ by considering the downstream expert calibration such
that the average cost of the calibrated optimizer $R\circ h_W$ is minimized.

By analogy to a multi-layer neural network, we can view the added expert calibrator as another
implicit ``layer'' following the ML predictions \cite{L2O_DifferentiableConvexOptimization_Brandon_NEURIPS2019_9ce3c52f,L2O_DifferentiableOptimization_Brandon_ICML_2017_amos2017optnet}.
Unlike in a standard neural network where
each layer has simple matrix computation and activation, the expert calibration layer itself
is an optimization algorithm.
Thus,
the training process of $R\circ h_W$ in \ouralg is much more challenging than that of a standalone ML-based optimizer used in the simple two-stage approach (i.e., $R\oplus h_W$).

The choices of the expert calibrator $R$
and loss function for training $R\circ h_W$ are critical
to achieve our goal of reducing
the average cost with good performance robustness.
Next, we present the details of our expert calibrator and loss function.

\subsection{Differentiable Expert Calibrator}\label{sec:calibrator_main}

We now design the expert calibrator that has good performance robustness and meanwhile is {differentiable} with respect to the ML predictions.

\subsubsection{The necessity of being differentiable}
When training an ML model,
 gradient-based backpropagation is arguably the state-of-the-art training method \cite{DNN_Book_Goodfellow-et-al-2016}.
Also, Section~\ref{sec:calibration_limitation_simple} has already highlighted the drawback of training $h_W$ as a standalone model to minimize the pre-calibration cost without being
aware of the downstream expert calibrator $R$.
Therefore, when optimizing the weight $W$ of $h_W$ to minimize the post-calibration cost, we need to back propagate the gradient of $R$
with respect to the ML prediction output of $h_W$ \cite{L2O_DifferentiableConvexOptimization_Brandon_NEURIPS2019_9ce3c52f,L2O_DifferentiableOptimization_Brandon_ICML_2017_amos2017optnet}.
Without the gradient, one may instead want to
 use some black-box gradient estimators like zero-order methods \cite{zero-order_opt_signalprocessing_ML_liu2020primer}.
 Nonetheless, zero-order methods are
 typically computationally expensive due to the many samples needed to estimate the gradient, especially in our recurrent structure where the base optimizers  are dependent on each other through time as illustrated in Fig.~\ref{fig:illustration}.
  Alternatively,
 one might want to pre-train a DNN to approximate the expert calibrator and then calculate the DNN's gradients as substitutes of the expert calibrator's gradients. Nonetheless, this alternative has its own limitation as well:
many samples are required to pre-train the DNN to approximate the expert calibrator and, more critically,
the gradient estimation error can  still be large  because a practical DNN only has finite generalization capability \cite{ML_Book_foundation_ML_mohri2018foundations}.
For these reasons, we would like our expert calibrator to be differentiable with respect to the ML predictions,
in order to apply backpropagation and efficiently train $h_W$ while being aware of
the downstream expert calibrator.

\subsubsection{Expert calibrator}\label{sec:expert_calibrator}
 Our design of the expert calibrator is highly relevant to the emerging
ML-augmented
algorithm designs
in the context of
online convex optimization with switching costs
 \cite{SOCO_OnlineMetric_UntrustedPrediction_ICML_2020_DBLP:conf/icml/AntoniadisCE0S20,SOCO_OnlineOpt_UntrustedPredictions_Switching_Adam_arXiv_2022}.
Commonly, the goal of ML-augmented algorithms is  to  achieve good
competitive ratios
in two cases:
 a bounded competitive ratio
even when the ML predictions are bad (i.e.,
robustness), while being close
to the ML predictions
when they already perform optimally (i.e., consistency).
Readers are referred to \cite{SOCO_OnlineOpt_UntrustedPredictions_Switching_Adam_arXiv_2022} for formal definitions of robustness and consistency.
Nonetheless,
 the existing algorithms \cite{SOCO_OnlineMetric_UntrustedPrediction_ICML_2020_DBLP:conf/icml/AntoniadisCE0S20,SOCO_OnlineOpt_UntrustedPredictions_Switching_Adam_arXiv_2022,SOCO_ML_ChasingConvexBodiesFunction_Adam_UnderSubmission_2022}
 view the ML-based optimizer
as an exogenous black box without addressing how the ML models are trained.
Moreover, they
 focus on switching costs in a metric space and hence do
not apply to our setting where the switching cost is defined in terms of the
squared Mahalanobis distance.
For these reasons, we propose a new expert calibrator
that
generalizes the state-of-the-art
expert algorithm --- Regularized OBD (R-OBD) which
matches the lower bound of any online
 algorithm for our problem setting \cite{SOCO_OBD_R-OBD_Goel_Adam_NIPS_2019_NEURIPS2019_9f36407e} --- by incorporating ML predictions.

The key idea of R-OBD is to regularize
the online actions by encouraging them
to stay close
to the actions that minimize the current hitting cost at each time step. Neither future contextual
information nor offline optimal actions are available in R-OBD. Thus, it uses
the minimizer of the current hitting cost as an estimate of the offline optimal action to regularize
the online actions, but clearly the minimizer
of the current hitting cost can only provide
limited information about the offline optimal since it ignores the switching cost.
Thus, in addition to the minimizer of the current hitting cost, we view the ML predictions
as another, albeit noisy, estimate of the offline optimal action.
Our expert calibrator, called \ouralgcal
(ML-Augmented R-OBD) and described in Algorithm~\ref{alg:RP-OBD},
generalizes R-OBD by introducing
the ML predictions as another regularizer.

\begin{algorithm}[!t]
\caption{Machine Learning Augmented R-OBD (\ouralgcal)}
\begin{algorithmic}[1]\label{alg:RP-OBD}
\REQUIRE $0<\lambda_1\leq1$, $\lambda_2\geq0$, $\lambda_3\geq0$, the initialized action $x_0$.
\STATE for $t=1,\cdots, T$
\STATE \quad Receive the context $y_t$.
\STATE \quad $v_t \leftarrow \arg\min_{x\in\mathcal{X}} f(x, y_t)$ \texttt{//Minimizer of the current hitting cost}
\STATE \quad $\tilde{x}_t \leftarrow h_W(x_{t-1}, y_t)$  \texttt{//ML prediction}
\STATE \quad $x_t \leftarrow \arg \min_{x\in\mathcal{X}} f(x, y_t) + \lambda_1 c(x, x_{t-1}) + \lambda_2 c(x, v_t) + \lambda_3 c(x, \tilde{x}_t)$ \texttt{//Calibrating the ML prediction}
\end{algorithmic}
\end{algorithm}

In \ouralgcal, at each time step $t$, the online action $x_t$ is chosen to minimize the weighted sum of the hitting cost, switching cost, the gap towards the current hitting cost's minimizer $v_t$, and the newly added gap towards the ML prediction $\tilde{x}_t$.  \revise{Compared to R-OBD, \ouralgcal introduces
a new regularizer $\lambda_3 c(x, \tilde{x}_t)$ in Line~5 of Algorithm~\ref{alg:RP-OBD} that keeps the online actions close to the ML predictions --- the greater $\lambda_3\geq0$, the stronger regularization.}
Thus, \ouralgcal is more general than R-OBD, with tunable weights $\lambda_1$, $\lambda_2$ and $\lambda_3$  to balance the objectives. We also use $R_{\lambda}$ to represent \ouralgcal to highlight that it is parameterized by $\lambda=(\lambda_1,\lambda_2,\lambda_3)$.

 More precisely, when $\lambda_2>0$ and $\lambda_3=0$,
\ouralgcal reduces to the standard R-OBD;
when  $\lambda_1=1$ and $\lambda_2=\lambda_3=0$, \ouralgcal becomes the simple greedy algorithm which minimizes the current total cost (including the hitting and switching costs) at each step;
and when $\lambda_2=0$ but $\lambda_1=\lambda_3=1$,
\ouralgcal will become the Follow the Prediction (FtP) algorithm studied in \cite{SOCO_OnlineMetric_UntrustedPrediction_ICML_2020_DBLP:conf/icml/AntoniadisCE0S20}.

\subsubsection{Competitive ratio} Next, to show the advantage of \ouralgcal in terms of the competitive ratio compared to pure ML predictions, we provide a competitive ratio upper bound in the following theorem,
whose proof can be found in Appendix~\ref{appendix:proof_cr_robd}.

\begin{theorem}\label{thm:cr_robd}
Assume that the hitting cost function $f$ is $m$-strongly convex, and
that the switching cost function $c$ is the Mahalanobis distance by matrix $Q$ with the minimum
and maximum eigenvalues being $\frac{\alpha}{2}$ and $\frac{\beta}{2}$, respectively. If the ML
predictions are $\rho$-accurate (Definition~\ref{def:predacc}),
\ouralgcal has a competitive ratio upper bound of $\max \left(\frac{m+ \lambda_2 \beta }{m\lambda_1}, 1+\frac{\beta^2}{\alpha} \cdot \frac{\lambda_1}{(\lambda_2+\lambda_3)\beta + m}\right) + \frac{\lambda_3 \beta}{2\lambda_1}\rho$. Moreover, by optimally setting $\lambda_2=\frac{m\lambda_1}{2\beta}\left(\sqrt{\left(1+\frac{\beta}{m}\theta\right)^2+\frac{4\beta^2}{\alpha m}}+1-\frac{2}{\lambda_1}-\frac{\beta}{m}\theta\right)$ with $\theta=\frac{\lambda_3}{\lambda_1}$ being the trust parameter, the competitive ratio upper bound of \ouralgcal becomes
$1 + \frac{1}{2}
\left[\sqrt{(1+\frac{ \beta }{m}\theta)^2 + \frac{4 \beta^2}{m \alpha}} - \left(1 + \frac{\beta}{m}\theta)\right)\right] + \frac{\beta}{2} \theta\cdot\rho$.
\end{theorem}

Theorem~\ref{thm:cr_robd} provides a dimension-free upper bound of the competitive ratio achieved by \ouralgcal. Interestingly, although
we have three weights $\lambda_1,\lambda_2$
and $\lambda_3$, the optimal upper bound only depends on $\theta=\frac{\lambda_3}{\lambda_1}$, which we refer to as the trust parameter describing how much we trust the ML predictions.

In particular,
when we completely distrust and ignore
the ML predictions (i.e.,
 $\theta=\frac{\lambda_3}{\lambda_1}=0$), we can recover the competitive ratio of standard R-OBD as $\frac{1}{2} (\sqrt{1 + \frac{4 \beta^2}{\alpha m} }  + 1)$ by setting $\lambda_2=\frac{m\lambda_1}{2\beta}\left(\sqrt{1+\frac{4\beta^2}{\alpha m}}+1-\frac{2}{\lambda_1}\right)$, which  also matches the lower bound of the competitive ratio for any online algorithm \cite{SOCO_OBD_R-OBD_Goel_Adam_NIPS_2019_NEURIPS2019_9f36407e}. When we have full trust
on the ML predictions (i.e., $\theta=\frac{\lambda_3}{\lambda_1}\to\infty$),
the upper bound also increases to infinity unless the ML predictions are perfect with $\rho=0$.
This is consistent with our discussion of
the limitations of purely using ML predictions
in Section~\ref{sec:ML_limitation}.

When we partially trust the ML predictions (i.e., $0<\theta<\infty$), we can see that the upper bound
increases with the prediction error $\rho$. In particular,
when the ML predictions are extremely bad (i.e., $\rho\to\infty$), the competitive ratio of \ouralgcal can also become unbounded.
\revise{The non-ideal result is not a bug of \ouralgcal or our proof technique, but rather due to the fundamental limit
for the problem we study --- as shown in Appendix~B of \cite{SOCO_OnlineOpt_UntrustedPredictions_Switching_Adam_arXiv_2022}, no
ML-augmented online algorithms could simultaneously achieve a consistency (i.e., competitive ratio for $\rho=0$) less than the optimal competitive ratio of R-OBD as well as a bounded robustness (i.e., competitive ratio for $\rho=\infty$).
In our case, \ouralgcal can benefit from good ML predictions and achieve a smaller competitive ratio than R-OBD when $\rho\to0$, and thus its unbounded competitive ratio for $\rho=\infty$ is anticipated.
Nonetheless, it is an interesting future problem
to consider the other alternative, i.e.,
designing a new expert calibrator that achieves a bounded competitive
ratio when $\rho=\infty$ without being able to
exploit the benefit of good ML predictions when $\rho\to0$.
}

More importantly,
by properly setting the trust parameter $\theta=\frac{\lambda_3}{\lambda_1}$, the competitive ratio upper bound of \ouralgcal
can still be much smaller than the competitive ratio \emph{lower} bound of purely using ML predictions (Lemma~\ref{thm:ml_eta}) when the ML prediction error $\rho$ is sufficiently large. This means that when the ML predictions are of a low quality,
\ouralgcal is guaranteed to be able to calibrate the predictions for significantly improving the competitive ratio.

To further illustrate this point,
 we show in Fig.~\ref{fig:cr_bound} the competitive ratio lower bound  of the pure ML-based optimizer $h$,  the competitive ratio
 upper bounds of \ouralgcal, and that of R-OBD under various prediction errors $\rho$.  We can find that when the prediction error $\rho$ is small enough, the pure ML predictions are the best. But, this is certainly too opportunistic for
a practical ML-based optimizer, whose actual testing performance cannot
always be very good (i.e., $\rho$
is sufficiently small) due to the fundamental limitations discussed in Section~\ref{sec:ML_limitation}.
On the other hand, when the prediction error $\rho$ is large, expert calibration via \ouralgcal is guaranteed to be better than the un-calibrated ML predictions. Additionally, introducing ML predictions with large errors in \ouralgcal can increase the competitive ratio upper bound compared to the  pure R-OBD expert algorithm that is independent of the prediction error.
Nonetheless, by lowering the trust parameter,
even ML predictions with errors up to a certain threshold
are still beneficial to R-OBD and lead to a lower competitive ratio upper bound than R-OBD.

 \begin{wrapfigure}[11]{r}{0.4\textwidth}
\vspace{-0.2cm}
\includegraphics[width={0.39\textwidth}]{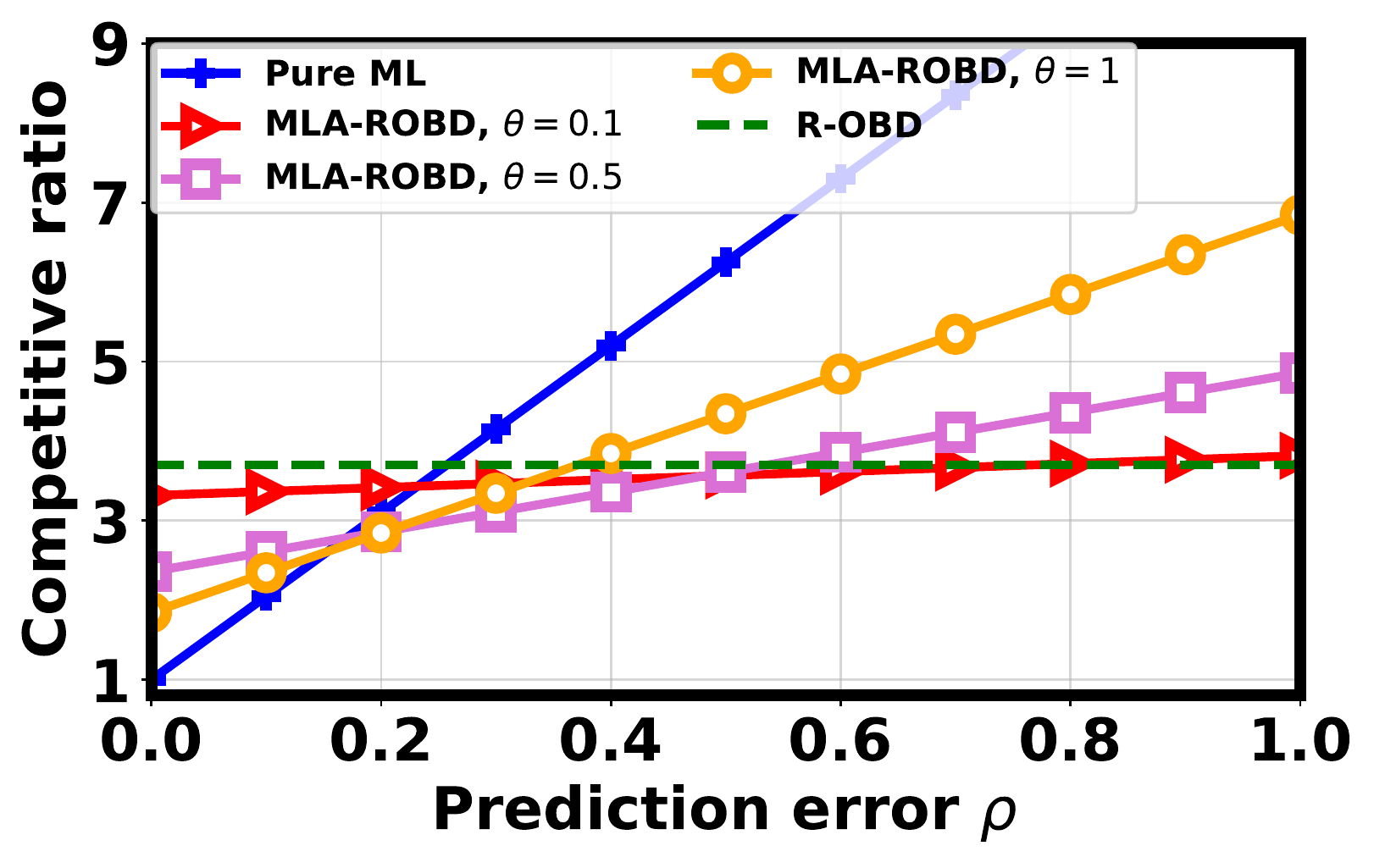}
\vspace{-0.5cm}
	\caption{Illustration of competitive ratios of different algorithms.}
\label{fig:cr_bound}
\end{wrapfigure}
In summary,
 when the prediction error is not too large,
 the competitive ratio upper bound of \ouralgcal can be even smaller than that of R-OBD;
 when the prediction error is large enough,
 \ouralgcal will be worse than R-OBD in terms
 of the competitive ratio upper bound due to
 the introduction of erroneous ML predictions,
 but,  it can still be guaranteed to be much better than purely using ML predictions by setting a proper trust parameter. 
 Based on these insights, we shall design our loss function for training the ML model $h_W$
 in our expert-calibrated optimizer $R_{\lambda}\circ h_W$.
 
\revise{
\subsubsection{Regret}
While we focus on the competitive ratio analysis,
another metric commonly considered in the literature
 is the regret, which measures the difference between the cost achieved by an online algorithm and that by an oracle \cite{SOCO_OBD_R-OBD_Goel_Adam_NIPS_2019_NEURIPS2019_9f36407e}. Crucially, a desired result
 is that the regret grows sublinearly in time $T$, i.e., the time-averaged cost difference between an online algorithm and an oracle approaches zero as $T\to\infty$.
 Without ML predictions, the standard
 R-OBD can
 simultaneously achieve a sublinear regret while
 providing a dimension-free and bounded competitive ratio \cite{SOCO_OBD_R-OBD_Goel_Adam_NIPS_2019_NEURIPS2019_9f36407e}.
 Thus, it is interesting to see
 if \ouralgcal, which incorporates ML predictions into R-OBD, can also achieve the same.
To this end,
 we consider the $L$-constrained regret which, defined as follows, generalizes the classic static regret.

\begin{definition}[$L$-constrained Regret] For any problem instance with context $\bm{s}=(x_0,\bm{y})$,
suppose that an online algorithm 
makes actions $x_{1:T}=(x_1,x_2,\cdots,x_T)$
and has a total cost of 
$cost(x_{1:T})=\sum_{t = 1}^T f(x_t, y_t) + c(x_t, x_{t-1})$.
 The
 $L$-constrained regret
is defined
as
\begin{equation}
\label{eqn:L_constrained_regret}
Regret(x_{1:T},x_{1:T}^L)=
cost(x_{1:T})- cost(x_{1:T}^L),
\end{equation}
where $x_{1:T}^L$
denotes the actions made by the
$L$-constrained optimal oracle  that
solves $x_{1:T}^L=\arg\min_{x_{1:T}} \sum_{t = 1}^T f(x_t, y_t) + c(x_t, x_{t-1})$ subject to $\sum_{t=1}^T c(x_t, x_{t-1}) \leq L$
and $cost(x_{1:T}^L)$ is the corresponding total cost.
\end{definition}

In the $L$-constrained regret,
the oracle's total switching cost is essentially constrained by $L$.
When $L=0$, the $L$-constrained regret reduces to the classic static regret
where the oracle is only allowed to make one static action;
when $L$ increases, the oracle has more flexibility, making
the $L$-constrained regret closer to the fully dynamic regret \cite{SOCO_OBD_R-OBD_Goel_Adam_NIPS_2019_NEURIPS2019_9f36407e}. As the regret
compares an online algorithm against an $L$-constrained oracle,
we also introduce $(L, L_\rho)$-constrained prediction to quantify the ML prediction quality as follows.

\begin{definition}[$(L, L_\rho)$-constrained Prediction]\label{def:predacc_L}
The ML predictions, denoted as $\tilde{x}_1,\cdots \tilde{x}_T$, are said to be $(L, L_\rho)$-constrained if 
the following is satisfied:
\begin{equation}\label{eqn:rho_L_accurate}
\sum_{t=1}^T \|\tilde{x}_t-x_t^L \|^2 \leq L_\rho,
\end{equation}
where $\|\cdot\|$ is the $l_2$-norm, and $x_{1}^L,\cdots,x_T^L$ are the actions made by the $L$-constrained optimal oracle.
\end{definition}

Note that, unlike $\rho$-accurate prediction  (Definition~\ref{def:predacc}) that
is scale-invariant and normalized with respect to the unconstrained oracle's
optimal cost, the notion of $(L, L_\rho)$-constrained prediction provides a different characterization of ML prediction quality by
quantifying the absolute squared errors
between ML predictions and the $L$-constrained oracle's actions. Next,
we provide an upper bound on the $L$-constrained regret
in Theorem~\ref{thm:regret}, whose proof can be found in
Appendix~\ref{appendix:proof_regret_mla}.

\begin{theorem}
\label{thm:regret}
Assume that the hitting cost function $f$ is $m$-strongly convex, and
that the switching cost function $c$ is the Mahalanobis distance by matrix $Q$ with the minimum and maximum eigenvalues being $\frac{\alpha}{2}>0$ and $\frac{\beta}{2}$, respectively.
Additionally, assume that $\| Q x \|$ is bounded by $G < \infty$, the size/diameter of the feasible action set $\mathcal{X}$ is $\omega=\sup_{x,x'\in\mathcal{X}}\|x-x'\|$,
and $\lambda_1 \geq 1 - \frac{m}{4\beta}$.
For any problem instance and $(L,L_{\rho})$-constrained ML predictions,
 if 
 $\lambda_3 < \frac{\alpha}{\beta}$ and the total
 switching cost satisfies $\sum_{t=1}^T c(x_t, x_{t-1}) \leq L$,
 the $L$-constrained regret of \ouralgcal satisfies $Regret(x_{1:T},x_{1:T}^L)\leq\frac{\alpha}{\alpha - \lambda_3 \beta} \big( (\lambda_1 + \frac{m}{2\beta} )G \sqrt{\frac{2TL}{\alpha}} + \lambda_2 \frac{\beta T \omega^2}{2} + \frac{\lambda_3 \beta}{2} L_\rho \big)$; otherwise, 
 we have $Regret(x_{1:T},x_{1:T}^L)\leq(\lambda_1 + \frac{m}{2\beta} )G \sqrt{\frac{2TL}{\alpha}} + (\lambda_2 +\lambda_3)  \frac{\beta T \omega^2}{2} $.
 Moreover, by setting
  $\lambda_2 = \eta_2 (T, L, \omega, G)$
  and $\lambda_3 = \eta_3 (T, L, \omega, G)$ such that $\lim_{T \rightarrow \infty} \left(\eta_2 (T, L, \omega, G) + \eta_3 (T, L, \omega, G) \right) \frac{\omega^2}{G}\sqrt{\frac{T}{L}} < \infty$, we have $Regret(x_{1:T},x_{1:T}^L)=O(G\sqrt{TL})$.
\end{theorem}

In Theorem~\ref{thm:regret}, we see
that the $L$-constrained regret upper bound is increasing
in both of the regularization weights
 $\lambda_2\geq0$ and $\lambda_3\geq0$.
That is, by keeping the actions closer to  the minimizer $v_t$ of the pure hitting cost and/or ML predictions $\tilde{x}_t$, the worst-case regret bound also increases. On the other hand,
by Theorem~\ref{thm:cr_robd},
 regularization by properly setting
 $\lambda_2>0$ and $\lambda_3>0$ is necessary to improve the competitive ratio
 performance and exploit the benefit of good ML predictions.
 Thus,
Theorem~\ref{thm:regret} does not necessarily guarantee a sublinear regret
when we aim at achieving the optimal competitive ratio. In fact,
even without ML predictions, achieving
both the optimal competitive ratio and a sublinear regret remains an open question \cite{SOCO_OBD_R-OBD_Goel_Adam_NIPS_2019_NEURIPS2019_9f36407e}. 
 Nonetheless, Theorem~\ref{thm:regret} does show that
by choosing sufficiently small $\lambda_2\geq0$ and $\lambda_3\geq0$,
\ouralgcal can achieve a sublinear regret $O(G\sqrt{TL})$.
Together with
the result in Theorem~\ref{thm:cr_robd},
\ouralgcal also simultaneously achieves
 a dimension-free and conditionally-constant competitive ratio (conditioned on a finite prediction error $\rho$).
Note finally that by setting $\lambda_3=0$, our result on the $L$-constrained regret reduces to the one for  standard R-OBD \cite{SOCO_OBD_R-OBD_Goel_Adam_NIPS_2019_NEURIPS2019_9f36407e}.

}

\subsection{Loss Function for Training}

In Section~\ref{sec:calibrator_main}, we have described our differentiable expert calibrator
\ouralgcal and shown its key advantage --- improving performance robustness by lowering the worst-case competitive ratio of pure ML predictions. Besides the worst-case performance, another crucial goal is to achieve a low \emph{average} cost, for which
training the ML-based optimizer $h_W$ by explicitly considering the downstream calibrator $R_{\lambda}$ is essential (Section~\ref{sec:calibration_limitation_simple}).

As illustrated in Fig.~\ref{fig:illustration},
at each time step $t=1,\cdots,T$,
the calibrated action $x_t$ and its un-calibrated ML prediction  $\tilde{x}_t$ can be
expressed as
 \begin{equation}\label{eqn:calibratedopt}
 x_t=R_{\lambda}(y_t,x_{t-1},\tilde{x}_t) \;\;\text{ and  }\;\; \tilde{x}_t=h_W(y_t,x_{t-1}),
 \end{equation}
 where $R_{\lambda}$ is the expert calibrator (i.e., \ouralgcal) and $h_W$ is the ML-based optimizer.

Naturally,  our loss function for training
should explicitly consider the expert-calibrated
action $x_t$ to minimize the average cost, because
it is $x_t$, rather than the un-calibrated ML prediction $\tilde{x}_t$, that is actually being used by the online agent and directly determines the cost. Thus, our loss function includes
$\mathrm{cost}(R_{\lambda}\circ h_W,\bm{s})$
to supervise the training process.
On the other hand, if we only minimize
$\mathrm{cost}(R_{\lambda}\circ h_W,\bm{s})$ as our loss function, the ML-based optimzier $h_W$
is solely focusing on optimizing its un-calibrated prediction $\tilde{x}_t$ such that the post-calibration prediction $x_t$ achieves a low average cost. In other words, the un-calibrated
prediction $\tilde{x}_t$ could be very \emph{bad}
if we were to directly use $\tilde{x}_t$ as the agent's action.
By Theorem~\ref{thm:cr_robd}, the competitive ratio upper bound achieved by \ouralgcal
linearly increases with respect to the ML prediction error $\rho$. This means that, if the
pure ML predictions
$\tilde{x}_t$ are of a very low quality and have a
large prediction error, then the resulting competitive ratio of the expert-calibrated optimizer
$R_{\lambda}\circ h_W$ can also be very large, compromising the worst-case performance robustness.

To address this issue and make the calibrator more useful, we define another loss for the ML-based optimizer as follows
\begin{equation}\label{eqn:loss_new}
l(h_W,\bm{s})=\mathrm{relu}\left(\frac{\sum_{t=1}^T\|\tilde{x}_t-x^*_t\|^2}{\mathrm{cost}(\pi^*,\bm{s})}-\bar{\rho}\right),
\end{equation}
where $\mathrm{relu}(x)=\max(x,0)$, $\tilde{x}_t$ is the output of $h_W$,
$x^*_t$ is the offline optimal oracle's action
and $\bar{\rho}$ is a threshold of prediction error to determine whether a prediction is good enough.
The added loss essentially regularizes the ML-based optimizer by encouraging it to predict actions with low prediction errors.

By balancing the two losses via a hyperparameter
$\mu\in[0,1]$, our loss function for training the ML-based optimizer $h_W$ on a dataset $\mathcal{D}$
is given as follows:
\begin{equation}\label{eqn:unrolling}
	L_{\mathcal{D},R_{\lambda},\mu}(h_W) =\mu \frac{1}{|\mathcal{D}|}\sum_{\bm{s}\in \mathcal{D}}l(h_W,\bm{s}) +(1-\mu)\frac{1}{|\mathcal{D}|}\sum_{\bm{s}\in \mathcal{D}}\mathrm{cost}(R_{\lambda}\circ h_W,\bm{s}).
\end{equation}

For the training loss function in Eqn.~\eqref{eqn:unrolling}, by setting $\mu=1$, we recover the pure ML-based optimizer that can have a high cost when the testing input instances are far from the training instances. If $\mu=0$, we simply optimize the average post-calibration cost while ignoring the ML prediction error.
Most typically, we set $\mu\in(0,1)$ to restrain the ML prediction error for meaningful expert calibration, and also to reduce the average cost.

The training dataset can be constructed based on historical problem instances and also possibly enlarged via data augmentation. As we directly optimize the ML model weights to minimize the sum of hitting and switching costs, we do not need labels (i.e., the offline oracle's or an expert algorithm's actions).
Finally, note that we can also use held-out
validation dataset to tune the hyperparameters
(e.g., $\mu$, $\theta=\frac{\lambda_3}{\lambda_1}$, and learning rate) to achieve the most desired balance between the average cost and competitive ratio.

\subsection{Differentiating the Expert Calibrator}

We are now ready to derive the gradients of our differentiable expert calibrator \ouralgcal shown in Algorithm~\ref{alg:RP-OBD}. This is a crucial but non-trivial step for
two reasons: first, unlike an explicit layer with simple matrix computation along with an activation function  in a standard neural network,
the expert-calibrator \ouralgcal is essentially
an \emph{implicit} layer that executes an expert algorithm on its own \cite{DNN_ImplicitLayers_Zico_Website,L2O_DifferentiableConvexOptimization_Brandon_NEURIPS2019_9ce3c52f,L2O_DifferentiableOptimization_Brandon_ICML_2017_amos2017optnet};
and second, due to the recurrent structure for sequential decision making,
our backpropagation needs to be performed recurrently through time.

To efficiently train the ML-based optimizer $h_W$ in \ouralg,
we can apply various gradient-based algorithms such as
SGD and Adam. These algorithms all require
 backpropagation and hence gradients of the objective in Eqn.~\eqref{eqn:unrolling}.
  Thus, we need to differentiate $l(h_W,\bm{s})$ and $\mathrm{cost}(R_{\lambda}\circ h_W,\bm{s})$ with respect to the weight $W$ in the ML-based optimizer
  $h_W$.
  Next, we focus on
   the gradient of  $\mathrm{cost}(R_{\lambda}\circ h_W,\bm{s})$, while a similar method can be used to differentiate $l(h_W,\bm{s})$.

 By the expression of cost in Eqn.~\eqref{eqn:obj}, the gradient of $\mathrm{cost}(R_{\lambda}\circ h_W,\bm{s})$ given an input sequence $\bm{s}$ is written as
$$
\nabla_{W}\mathrm{cost}(R_{\lambda}\circ h_W,\bm{s})=\sum_{t=1}^T \nabla_{x_t}\left(f(x_t(W),y_t)+c(x_t(W)-x_{t-1}(W))\right)\nabla_{W}R_{\lambda}(y_t,x_{t-1}(W),\tilde{x}_t(W))
$$ where the expert-calibrated action $x_t$ and
the un-calibrated ML prediction $\tilde{x}_t$ are written as $x_t(W)$ and $\tilde{x}_t(W)$, respectively, to emphasize their dependence on the ML model weight $W$. Then, the gradient of the expert-calibrated action $x_t(W)$ with respect to $W$ is further written as
\begin{equation}\nonumber
\begin{split}
&\nabla_{W}R_{\lambda}(y_t,x_{t-1}(W),\tilde{x}_t(W))\\
=&\nabla_{x_{t-1}}R_{\lambda}(y_t,x_{t-1}(W),\tilde{x}_t(W))\nabla_{W}x_{t-1}(W)+\nabla_{\tilde{x}_{t}}R_{\lambda}(y_t,x_{t-1}(W),\tilde{x}_t(W))\nabla_{W}\tilde{x}_t(W),
\end{split}
\end{equation}
where $\nabla_{W}x_{t-1}(W)=\nabla_{W}R_{\lambda}(y_{t-1},x_{t-2}(W),\tilde{x}_{t-1}(W))$ is the gradient of the expert-calibrated action with respect to $W$ at step $t-1$ and $\nabla_{W}\tilde{x}_t(W)=\nabla_{W}h_W(y_{t-1},x_{t-1}(W))$ is the gradient of the un-calibrated ML prediction.

Now, it remains to differentiate the implicit layer of the expert calibrator. We cannot directly differentiate the calibrator.
But, since \ouralgcal solves a convex optimization problem at each time step, we instead differentiate it by its optimum condition shown as follows
\begin{equation}
\nabla_{x_t} f(x_t, y_t) + \lambda_1 \nabla_{x_t} c(x_t, x_{t-1}) + \lambda_2 \nabla_{x_t} c(x_t, v_t) + \lambda_3 \nabla_{x_t} c(x_t, \tilde{x}_t) = 0.
\end{equation}
By taking the gradients for both sides and solving the obtained equation, we can derive the gradients of $R_{\lambda}$ with respect to $\tilde{x}_{t}$ and $x_{t-1}$ as
\begin{equation}\label{eqn:gradientml}
\nabla_{\tilde{x}_t}R_{\lambda}(y_t,x_{t-1},\tilde{x}_t)=-\lambda_3 Z_t^{-1} \nabla_{\tilde{x}_t,x_t}c(x_t,\tilde{x}_t),
\end{equation}
\begin{equation}\label{eqn:gradientprior}
\nabla_{x_{t-1}}R_{\lambda}(y_t,x_{t-1},\tilde{x}_t)=-\lambda_1 Z_t^{-1} \nabla_{x_{t-1},x_t}c(x_t,x_{t-1}),
\end{equation}
where $Z_t=\nabla_{x_t,x_t} \left[f(x_t, y_t) + \lambda_1  c(x_t, x_{t-1}) + \lambda_2  c(x_t, v_t) + \lambda_3c(x_t, \tilde{x}_t)\right]$. Note that for our switching cost $c(x_t,x_{t-1})=\left(x_t-x_{t-1}\right) ^\top Q \left(  x_t-x_{t-1}\right)$
defined in terms of
the Mahalanobis distance with respect to  $Q\in\mathcal{R}^{d\times d}$, we have
 $\nabla_{\tilde{x}_t,x_t}c(x_t,\tilde{x}_{t})=\nabla_{x_{t-1},x_t}c(x_t,x_{t-1})=-2Q$ and $\nabla_{x_t,x_t}c(x_t,x_{t-1})=\nabla_{x_t,x_t}c(x_t, v_t)=\nabla_{x_t,x_t}c(x_t, \tilde{x}_t)=2Q$.
The details of deriving the gradients in Eqn.~\eqref{eqn:gradientml} and Eqn.~\eqref{eqn:gradientprior} are given in Appendix~\ref{sec:gradientproof}.
\section{Performance Analysis}

To conclude the design of \ouralg, we now analyze its performance in terms of
 the average cost as well as the tail cost ratio.

\subsection{Average Cost}\label{sec:analysis_average}
Average cost is a crucial metric and measures the performance of \ouralg in many typical cases.
Next, we analyze the average cost of \ouralg by the generalization theory of statistical learning \cite{ML_Book_foundation_ML_mohri2018foundations}. First, we define the optimal ML model weight $W^*$ in $h_W$ as follows.
\begin{definition}
The optimal model weight $W^*$ in the ML-based
optimizer $h_W$ of \ouralg is defined as
the minimizer of the expected training loss below:
\begin{equation}\label{eqn:expectedloss}
	W^*=\arg\min_{W\in\mathcal{W}}\mu \mathbb{E}\left[l(h_W,\bm{s})\right] +(1-\mu)\mathbb{E}\left[\mathrm{cost}(R_{\lambda}\circ h_W,\bm{s})\right],
\end{equation}
where $\mathcal{W}$ is the weight space depending
on the model $h_W$ (e.g., DNN with a certain  architecture) and the expectation is taken over the  environment distribution $\mathbb{P}$ of the input instance $s$.
\end{definition}

The optimal ML model, i.e., $h_{W^*}$, minimizes the
weighted sum of the expected average cost of the expert-calibrated optimizer and the expectation of the ML prediction error $\rho$ (Definition~\ref{def:predacc}) conditioned on $\rho\geq \bar{\rho}$.
By setting $0<\mu<1$,   $R_{\lambda}\circ h_{W^*}$ is essentially an online optimizer that minimizes the average cost while constraining the \emph{average} prediction error $\rho$ conditioned on $\rho\geq \bar{\rho}$.

Next, we denote the optimal weight $\hat{W}^*$ that minimizes the empirical loss function in Eqn.~\eqref{eqn:unrolling} as
$
\hat{W}^*=\arg\min_{W\in\mathcal{W}}L_{\mathcal{D},R_{\lambda},\mu}(h_W)
$. Assuming that $\hat{W}\in\mathcal{W}$ is the model weight after training on the loss function
in Eqn.~\eqref{eqn:unrolling}, the training error on the training dataset $\mathcal{D}$ is denoted as
\begin{equation}\label{eqn:training_error}
\mathcal{E}_{\mathcal{D}}=L_{\mathcal{D},R_{\lambda},\mu}(h_{\hat{W}})-L_{\mathcal{D},R_{\lambda},\mu}(h_{\hat{W}^*}).
\end{equation}
Now,
we bound the expected average cost in the next theorem, whose proof is in Appendix \ref{sec:proofavgcost}.
\begin{theorem}\label{thm:avgcostclb}
Assume that the dataset $\mathcal{D}$ is drawn from a distribution $\mathbb{P}'$ and the environment distribution is $\mathbb{P}$. If  $R_{\lambda}\circ{h}_{\hat{W}}$ is trained on the loss function in Eqn.~\eqref{eqn:unrolling}, then with probability at least $1-\delta, \delta\in(0,1)$, we have
	\begin{equation}\label{eqn:generalization_bound}
	\mathrm{AVG}(R_{\lambda}\circ{h}_{\hat{W}})\leq \mathrm{AVG}(R_{\lambda}\circ h_{W^*})+\frac{\mu}{1-\mu}\mathbb{E}\left[l(h_{W^*},\bm{s})\right]+\frac{1}{1-\mu}\mathcal{E}_{\mathcal{D}}+O \left( \sqrt{\frac{\log(1/\delta)}{|\mathcal{D}|}}+D(\mathbb{P},\mathbb{P}')\right),
	\end{equation}
where $\mu$ is the trade-off parameter in Eqn.~\eqref{eqn:expectedloss}, $\mathcal{E}_{\mathcal{D}}$ is the training error in Eqn.~\eqref{eqn:training_error}, $|\mathcal{D}|$ is the number of instances in the training dataset, and
$D(\mathbb{P},\mathbb{P}')$ measures the distributional discrepancy between the distribution $\mathbb{P}'$ from which
the training instances are drawn from and the environment distribution $\mathbb{P}$.
\end{theorem}

Theorem~\ref{thm:avgcostclb} quantifies the average cost gap between the expert-calibrated learning model $R_{\lambda}\circ{h}_{\hat{W}}$ and the optimal optimizer $R_{\lambda}\circ h_{W^*}$.
We can see that the  gap is affected by the training-induced  error
$\mathcal{E}_{\mathcal{D}}$
 and the generalization error induced by the distribution discrepancy between the training empirical distribution and the environment distribution $\mathbb{P}$.
Also,
 the average cost bound has an additional  term $\frac{\mu}{1-\mu}\mathbb{E}\left[l(h_{W^*},\bm{s})\right]$ which is caused by the first term in the training loss function in Eqn.~\eqref{eqn:unrolling}.
 When $\mu=0$, although the additional term disappears and $R_{\lambda}\circ h_{\hat{W}}$ purely minimizes the average cost,
 the downside is that the
  pre-calibration ML prediction error could be empirically very high, which thus leads to a high post-calibration competitive ratio upper bound (Theorem~\ref{thm:cr_robd}).
 On the other hand,
  when $\mu=1$, the average cost is neglected during the training process and hence can be unbounded, which also implies
  that minimizing the average prediction error
  (and hence {average} \emph{ratio} of the
  expert-calibrated optimizer's cost to the offline optimal cost) does not necessarily lead to a lower average cost.
In the bound in Eqn.~\eqref{eqn:generalization_bound},
 we also  see the term  $D(\mathbb{P},\mathbb{P}')$ due to the training-testing distributional discrepancy. When $D(\mathbb{P},\mathbb{P}')$ increases,
 although the average cost upper bound becomes larger, 
 the benefit of expert calibration can be more significant compared to purely using ML predictions without calibration.
 This will be empirically validated in the next section and also can be explained as follows.
 By viewing \ouralg as a two-layer model (i.e.,
 the ML-based optimizer layer followed by the expert calibration layer), we see that
 the first layer is trainable,  but the second layer is an expert algorithm that is specially designed
 for robustness and hence less vulnerable to the distributional shift.
 Thus, compared to
 a pure ML-based optimizer
 trained over the training distribution,
 \ouralg with an additional expert calibration layer will be affected less by
 distributional shifts.

\revise{Finally, we explain the necessity of holistically training the expert-calibrated optimizer $R_{\lambda}\circ h_{\hat{W}}$
 from the perspective of the generalization bound.
 For the sake of discussion, we assume $\mu=0$ in the loss function to focus
 on the average cost.
 If we
follow the  two-stage approach
(Section~\ref{sec:calibration_limitation_simple}) and train the ML-based optimizer as a standalone model with the weight $\tilde{W}$,
  we can derive a new  generalization
bound for the simple
optimizer $R\oplus h_{\tilde{W}}$.
The new generalization bound
takes the same form of Eqn.~\eqref{eqn:generalization_bound}
except for changing ML model weight from $\hat{W}$ to $\tilde{W}$.
 Specifically, the error term $\mathcal{E}_{\mathcal{D}}$ contained in the new
 bound
 becomes
$L_{\mathcal{D},R_{\lambda},\mu}(h_{\tilde{W}})-L_{\mathcal{D},R_{\lambda},\mu}(h_{\hat{W}^*})$,
where the ML model weight $\tilde{W}$
is trained separately to minimize
the empirical pre-calibration loss while
the term $L_{\mathcal{D},R_{\lambda},\mu}(h_{\tilde{W}})$ is evaluated based
on the empirical post-calibration  loss. On the other hand,
the training  error
$\mathcal{E}_{\mathcal{D}}$ defined in
Eqn.~\eqref{eqn:training_error}
contained in
  the generalization bound
  for $R_{\lambda}\circ{h}_{\hat{W}}$
  in Eqn.~\eqref{eqn:generalization_bound}
 can be made sufficiently small by using state-of-the-art training algorithms \cite{DNN_Book_Goodfellow-et-al-2016}.
Thus,
by simply concatenating
an ML-based optimizer $h_{\tilde{W}}$
with an expert calibrator $R$,
 the new optimizer
 $R\oplus h_{\tilde{W}}$
 will likely
 have a larger generalization bound
 than $R_{\lambda}\circ{h}_{\hat{W}}$
 due to the increased training error term.}

\subsection{Tail Cost Ratio}\label{sec:tail_cost_ratio}

For $\rho$-accurate ML predictions, the
competitive ratio of \ouralg by using
\ouralgcal as the expert calibrator follows
Theorem~\ref{thm:cr_robd}.
While \ouralg can lower the competitive
ratio of R-OBD with good ML predictions,
the worst-case competitive
ratio keeps increasing as the ML prediction error $\rho$ increases.
This result comes from
the fundamental limit for our problem setting
as shown in \cite{SOCO_OnlineOpt_UntrustedPredictions_Switching_Adam_arXiv_2022}.
 Thus, instead of considering the pessimistic worst-case competitive ratio of \ouralg, we resort to the high-percentile tail ratio of the cost achieved by \ouralg to that of the offline optimal oracle.
We simply refer to this ratio as the tail cost ratio, which
can provide a probabilistic view of the performance robustness of the algorithm.

\begin{theorem}\label{thm:tail_ratio}
Assume that the lower bound of the offline optimal cost is $\nu=\inf_{\bm{s}\in\mathcal{S}}\mathrm{cost}(\pi^*,\bm{s})$ and the size of the action set $\mathcal{X}$ is $\omega=\sup_{x,x'\in\mathcal{X}}\|x-x'\|$. Given the same assumptions as in Theorem~\ref{thm:avgcostclb}, with probability at least $1-\delta, \delta\in(0,1)$ regarding the randomness of input sequence $\bm{s}$,  we have
\begin{equation}\label{eqn:tail_cost_ratio_theorem}
\frac{\mathrm{cost}(R_{\lambda}\circ h_{\hat{W}},\bm{s})}{\mathrm{cost}(\pi^*,\bm{s})}\leq 1 + \frac{1}{2}
\left[\sqrt{(1+\frac{ \beta }{m}\theta)^2 + \frac{4 \beta^2}{m \alpha}} - \left(1 + \frac{\beta}{m}\theta)\right)\right] + \frac{\beta}{2} \theta\cdot\rho_{\mathrm{tail}},
\end{equation}
where $\alpha$, $\beta$, $m$ and $\theta$ are explained in Theorem \ref{thm:cr_robd}, and the tail ML prediction error
$\rho_{\mathrm{tail}}=\bar{\rho}+\mathbb{E}\left[l(h_{W^*},\bm{s})\right]+\frac{1-\mu}{\mu}\mathbb{E}\left[\mathrm{cost}(R_{\lambda}\circ h_{W^*},\bm{s})\right]+\frac{1}{\mu}\mathcal{E}_{\mathcal{D}}
    +O \left( \sqrt{\frac{\log(2/\delta)}{|\mathcal{D}|}}\right)+O\left(D(\mathbb{P},\mathbb{P}')\right)+O\left(\frac{\omega^2 \sqrt{T}}{\nu}\|\Gamma\|\sqrt{\frac{1}{2}\log(\frac{4}{\delta})}\right)$,
    in which $\mu$, $\mathcal{D}$, $\mathcal{E}_{\mathcal{D}}$ and $D(\mathbb{P},\mathbb{P}')$ are given in Theorem~\ref{thm:avgcostclb}, $\bar{\rho}$ is the prediction error threshold parameter in Eqn.~\eqref{eqn:loss_new},
   \revise{$\Gamma$ is the mixing matrix   with respect to
a partition
   of the random sequence $\chi =\left\{\|\tilde{x}_t-x_t^*\|^2/\mathrm{cost}(\pi^*,\bm{s}), t\in[T]\right\}$ \cite{Markov_concentration_paulin2015concentration}, and $\|\Gamma\|$ denotes the Frobenius norm.}
    \end{theorem}

The proof of Theorem~\ref{thm:tail_ratio} is available in Appendix \ref{sec:proof_tailratio}. The key insight
is that for the hyperparameter $\mu\in(0,1]$ that
balances the ML prediction error and the average cost in the training loss function in Eqn.~\eqref{eqn:unrolling}, the tail ML prediction error $\rho_{\mathrm{tail}}$ is bounded, which, by Theorem~\ref{thm:cr_robd},
also leads to a bounded tail cost ratio.
More precisely, as $\mu$ increases within
$(0,1]$, more emphasis is placed on the ML prediction error in the training loss function
in Eqn.~\eqref{eqn:unrolling}. As a result,
the tail ML prediction error also decreases,
resulting in a reduced tail cost ratio.
In the extreme case of $\mu=0$,  the ML prediction error is completely neglected in
the loss function in Eqn.~\eqref{eqn:unrolling}
during the training process. Thus, as intuitively expected, the tail cost ratio can be extremely large and unbounded.

\revise{In Theorem~\ref{thm:tail_ratio}, the ML prediction error $\rho_{\mathrm{tail}}$ is a crucial term affected by the size of action set $\omega=\sup_{x,x'\in\mathcal{X}}\|x-x'\|$, the lower bound of the offline optimal cost $\nu=\inf_{\bm{s}\in\mathcal{S}}\mathrm{cost}(\pi^*,\bm{s})$, and the
Frobenius norm of the mixing matrix $\Gamma$  due to  the McDiarmid's inequality for dependent variables \cite{Markov_concentration_paulin2015concentration}.
Among them, the mixing matrix $\Gamma$ is an upper diagonal matrix,
whose dimension is the length of the corresponding partition of $\chi$. The mixing matrix relies on the distribution of $\chi$ and its norm affects the
degree of concentration of the sequence $\chi$.
Specifically, a sequence with independent variables has a partition with the identity matrix as the mixing matrix, and
a uniformly ergodic Markov chain has a partition with the mixing matrix  $\Gamma_{i,j}\leq\epsilon^{\mathrm{relu}(j-i-1)}\mathds{1}(j\geq i)$ where $\epsilon\in [0,1)$ is a parameter determined by the partition \cite{Markov_concentration_paulin2015concentration}. Given a fixed set of
other terms, the less concentration of
the sequence $\chi$ (or higher $\|\Gamma\|$), the higher tail ML prediction error
as well as the higher tail cost ratio.}

By setting
 $\mu\in(0,1)$, we highlight the importance of considering the loss function in  Eqn.~\eqref{eqn:unrolling}
 that includes two different losses --- one for the ML-based optimizer to have
a low average prediction error, and the other one for the post-calibration ML calibrations to have
a low average cost.
 Albeit weaker than the worst-case competitive ratio, the result in Theorem~\ref{thm:tail_ratio}
is also crucial and provides
 a probabilistic  performance robustness of \ouralg.

\section{Case Study: Sustainable Datacenter Demand Response}

In this section, we conduct a case study and
perform simulations on the application of sustainable datacenter demand response to validate the design of \ouralg. Importantly,
our results demonstrate that \ouralg can achieve a lower average cost,
as well as an empirically lower competitive ratio, than several baseline algorithms.
This nontrivial result is due in great part to
our
differentiable expert calibrator and
loss function designed based on our theoretical analysis.

\subsection{Application}

Fueled the demand for sustainability, green renewables, such as wind and solar energy, have been increasingly adopted in today's power grids.
Nevertheless, since renewables are not as stable
as traditional energy sources and can fluctuate rapidly over time, the integration of renewables
into the power grid poses substantial challenges to meet the time-varying demand. Consequently, this calls for more demand response resources --- energy loads that can flexibly respond to the fluctuating supplies --- in order to balance the supply and demand at all times \cite{NREL_EnergySavingGoal,DoE_Demand_Response_2006}.

On the other hand,
warehouse-scale datacenters are traditionally
viewed as energy hogs, taking up a huge
load in power grids \cite{AdamWierman_DataCenterDemandResponse_Survey_IGCC_2014}.
Nonetheless, this negative view has been quickly changing in recent years. Specifically,
unlike traditional loads, datacenters
have great flexibility to adjust their energy demands using a wide range of control knobs at little or even no cost, such as turning off unused servers, spatially shifting workloads to elsewhere, deferring non-urgent workloads, and setting a higher cooling temperature \cite{DataCenterDemandResponse_Report_Berkeley}.
Compounded by the megawatt scale, these knobs can shed the datacenter energy consumption by a significant amount and make datacenters valuable demand response resources.
In fact, an increasingly larger
number of datacenters have already been actively participating in demand response programs that help stabilize and green the power grid \cite{LiuZhenhua_PricingDataCenter_DemandResponse_Sigmetrics_2014,Shaolei_Colo_EDR_OnlineMulti_eEnergy_2016}.

We consider the application of sustainable datacenter demand response to facilitate integration of
rapidly fluctuating renewables, including wind power and solar  power, into a micro power grid (a.k.a. microgrid) serving the data center.
More formally,
the wind power $P_{\mathrm{wind},t}$ and solar  power $P_{\mathrm{solar},t}$ at each time step $t$ both rely on the corresponding weather conditions. Here, we use empirical equations to model the wind and solar renewables. For wind power, the amount of energy generated at step $t$ is modeled
by using the equation in \cite{wind_power_sarkar2012wind} as
 $  P_{\mathrm{wind},t}=\frac{1}{2}\kappa_{\mathrm{wind}}\varrho A_{\mathrm{swept}} V_{\mathrm{wind},t}^3$,
where $\kappa_{\mathrm{wind}}$ is the conversion efficiency (\%) of wind power, $\varrho$ is the air density ($kg/m^3$), $A_{\mathrm{swept}}$ is the swept area of the turbine ($m^2$), and $V_{\mathrm{wind},t}$ is the wind speed ($kW/m^2$) at time step $t$. For solar power, the amount of energy generated at step $t$ is modeled based on \cite{solar_power_wan2015photovoltaic} as
 $   P_{\mathrm{solar},t}=\frac{1}{2}\kappa_{\mathrm{solar}} A_{\mathrm{array}} I_{\mathrm{rad},t}(1-0.05*(\mathrm{Temp}_t-25))$,
where $\kappa_{\mathrm{solar}}$ is the conversion efficiency (\%) of the solar panel,  $A_{\mathrm{array}}$ is the array area ($m^2$), and $I_{\mathrm{rad},t}$ is the solar radiation ($kW/m^2$) at time step $t$, and $\mathrm{Temp}_t$ is the temperature ($^{\circ}$C) at step $t$.
At time step $t$, the total energy generated by the renewables $P_{\mathrm{r},t}=P_{\mathrm{wind},t}+P_{\mathrm{solar},t}$.

We assume that at each step, the datacenter offers demand response
to compensate the microgrid's power shortage
$y_t=\max(P_{\mathrm{s},t}-P_{\mathrm{r},t},0)$
(i.e., the net amount of energy shedding needed from the datacenter), where $P_{\mathrm{s},t}$
is the power shortage before renewable integration.
The amount of energy load shedding offered by the
datacenter is the action $x_t$.
We model the hitting cost as the squared norm of the difference between the action $x_t$ and the context $y_t$, i.e. $f(x_t,y_t)=\frac{1}{2}\|x_t-y_t\|^2$,
which captures the quadratic cost incurred by the microgrid due to the actual supply-demand mismatch \cite{LiuZhenhua_PricingDataCenter_DemandResponse_Sigmetrics_2014}.
Additionally, we model
the switching cost by a scaled and squared norm of the difference between two consecutive actions, i.e. $c(x_t,x_{t-1})=\frac{\alpha}{2}\|x_t-x_{t-1}\|^2$. This cost comes from the changes of demand response knobs in the datacenter, e.g., setting different supply air temperature to adjust energy shedding \cite{sharma2005balance}.
The goal of sustainable datacenter demand response
is to make online actions to minimize
the total cost in Eqn.~\eqref{eqn:obj}.

\subsection{Simulation Settings}

We give the details about the settings in our simulation.
The switching cost parameter is set as $\alpha=10$.
To calculate the renewable energy, we set the parameters for wind power as $\kappa_{\mathrm{wind}}=30\%$, $\varrho=1.23 kg/m^3$, $A_{\mathrm{swept}}=500,000 m^2$. The wind speed  data is collected from the National Solar Radiation Database \cite{NSRDB_sengupta2018national}, which contains hourly data for the year of 2015. Additionally, we set the parameters for solar power as $\kappa_{\mathrm{wind}}=10\%$, $A_{\mathrm{array}}=10,000 m^2$. The temperature  data and the Global Horizontal Irradiance (GHI) data are also from the National Solar Radiation Database \cite{NSRDB_sengupta2018national}. For both temperature and GHI data, we use hourly data for the year of 2015.
A snapshot of the context sequence $y_t$ is shown in Fig.~\ref{fig:energy_trace}.

 \begin{wrapfigure}[]{r}{0.4\textwidth}
\vspace{-0.2cm}
\includegraphics[width={0.39\textwidth}]{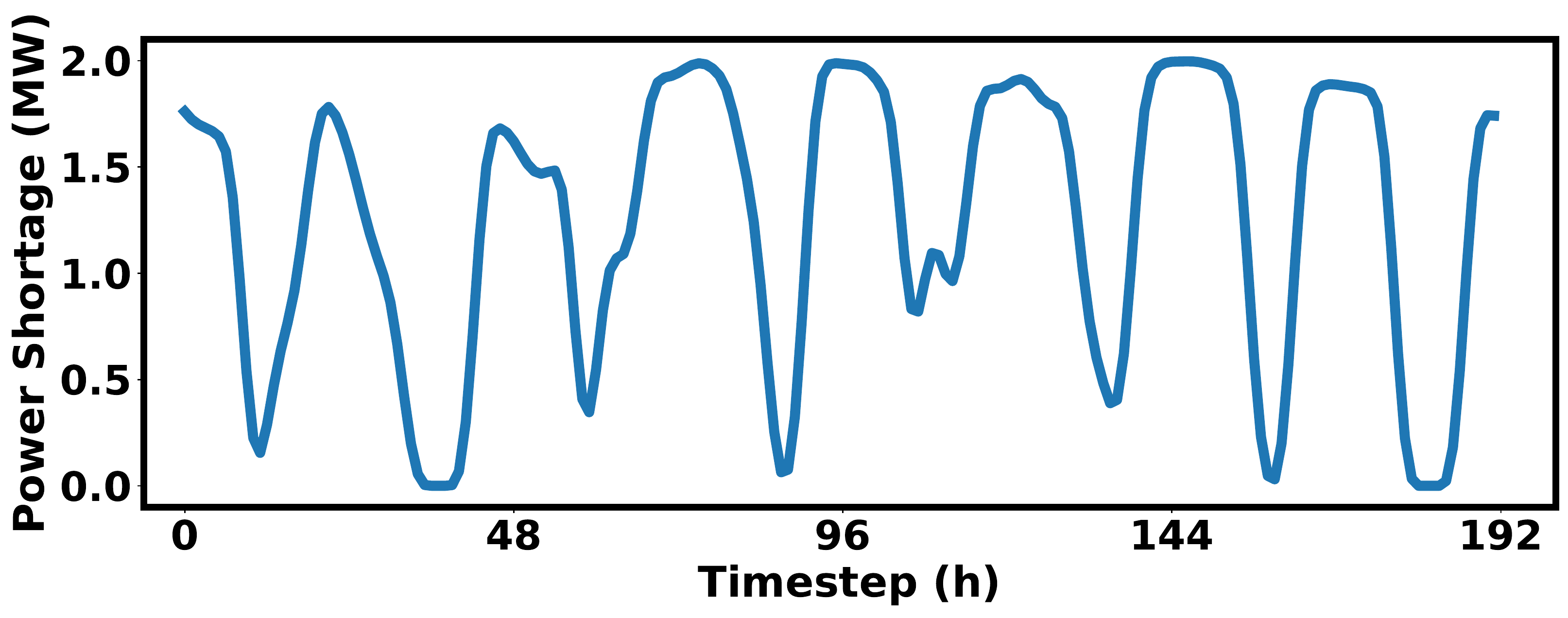}
\vspace{-0.3cm}
	\caption{A snapshot of the context $y_t$.}
\label{fig:energy_trace}
\end{wrapfigure}
 We use the hourly data of the first two months and data augmentation  to generate a training dataset with 1400 sequences, each with 24 hours. The data of the third  month is used for validation and hyperparameter tuning. The data of the remaining nine months  is used as the testing dataset.
 Each problem instance is $T=24$ hours.  We perform continuous testing by using the action of the previous testing instance as the initial action of the current testing instance.
 Fig.~\ref{fig:ood} visualizes
 the t-SNE of discrepancy between training and testing
 distributions \cite{tsne_visualize_2008}.
 We can see that
 the testing distribution shifts from the training distribution,
 which we refer to as out-of-distribution testing
 and is common in practice
 since the distribution of (finite) training samples may not truly reflect the testing distribution.
 \begin{figure*}[!t]	
	\centering
	\subfigure[Testing of Apr. to Jun.]{
		\label{fig:ood_1}
		\includegraphics[width=0.3\textwidth]{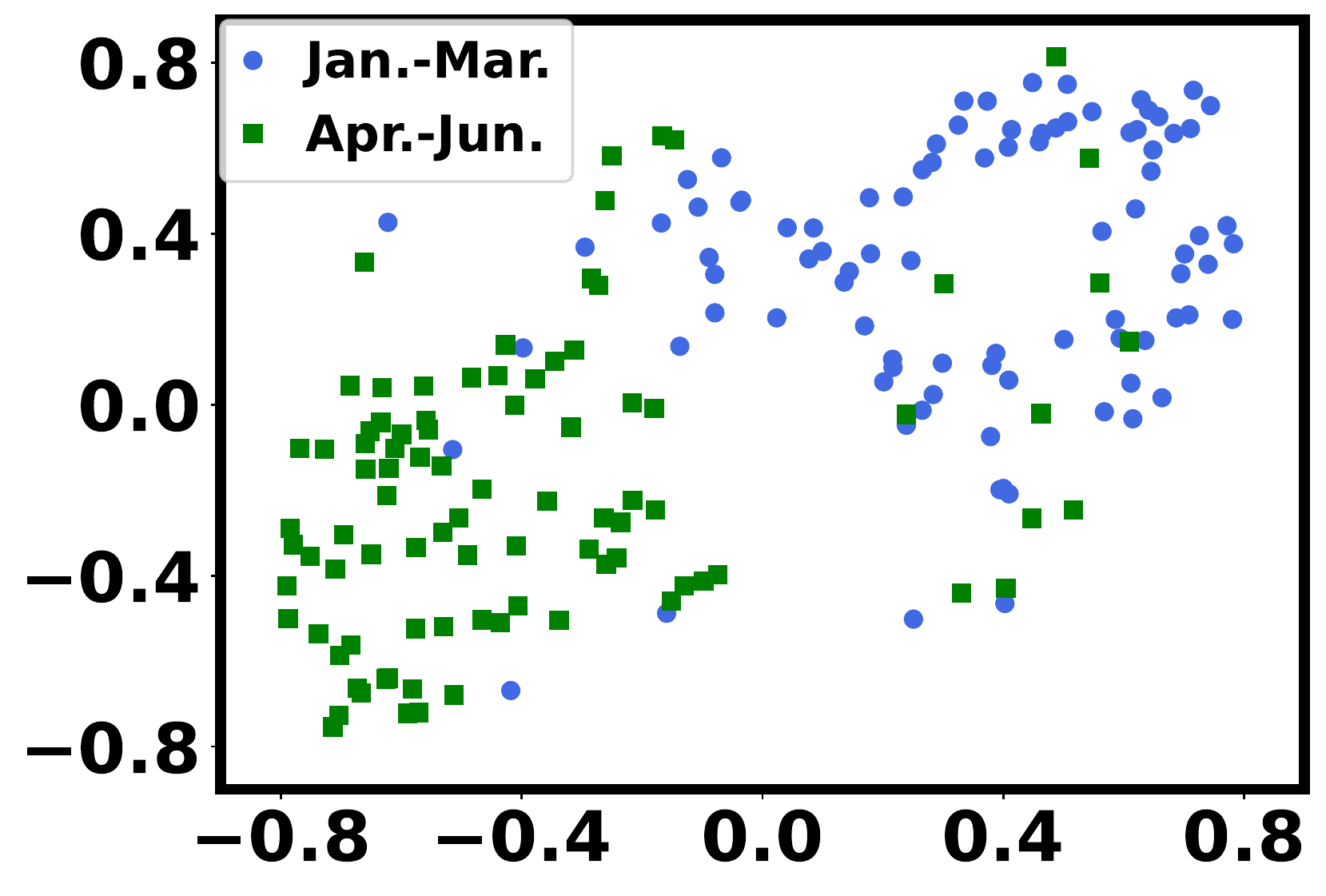}
	}
	\subfigure[Testing of Jul. to Sep.]{
		\label{fig:ood_2}
		\includegraphics[width=0.3\textwidth]{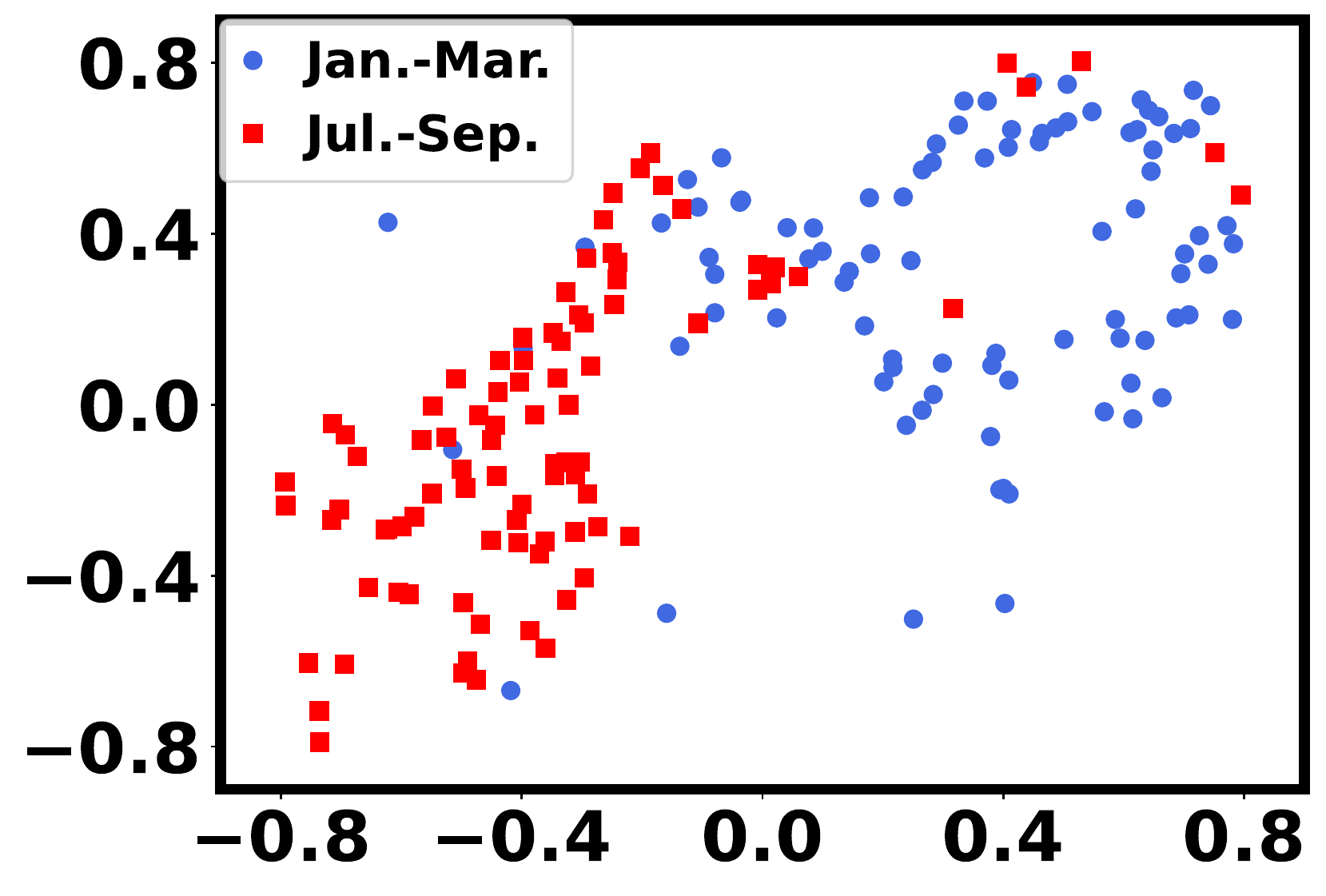}
	}%
	\subfigure[Testing of Oct. to Dec.]{
		\label{fig:ood_3} \includegraphics[width={0.3\textwidth}]{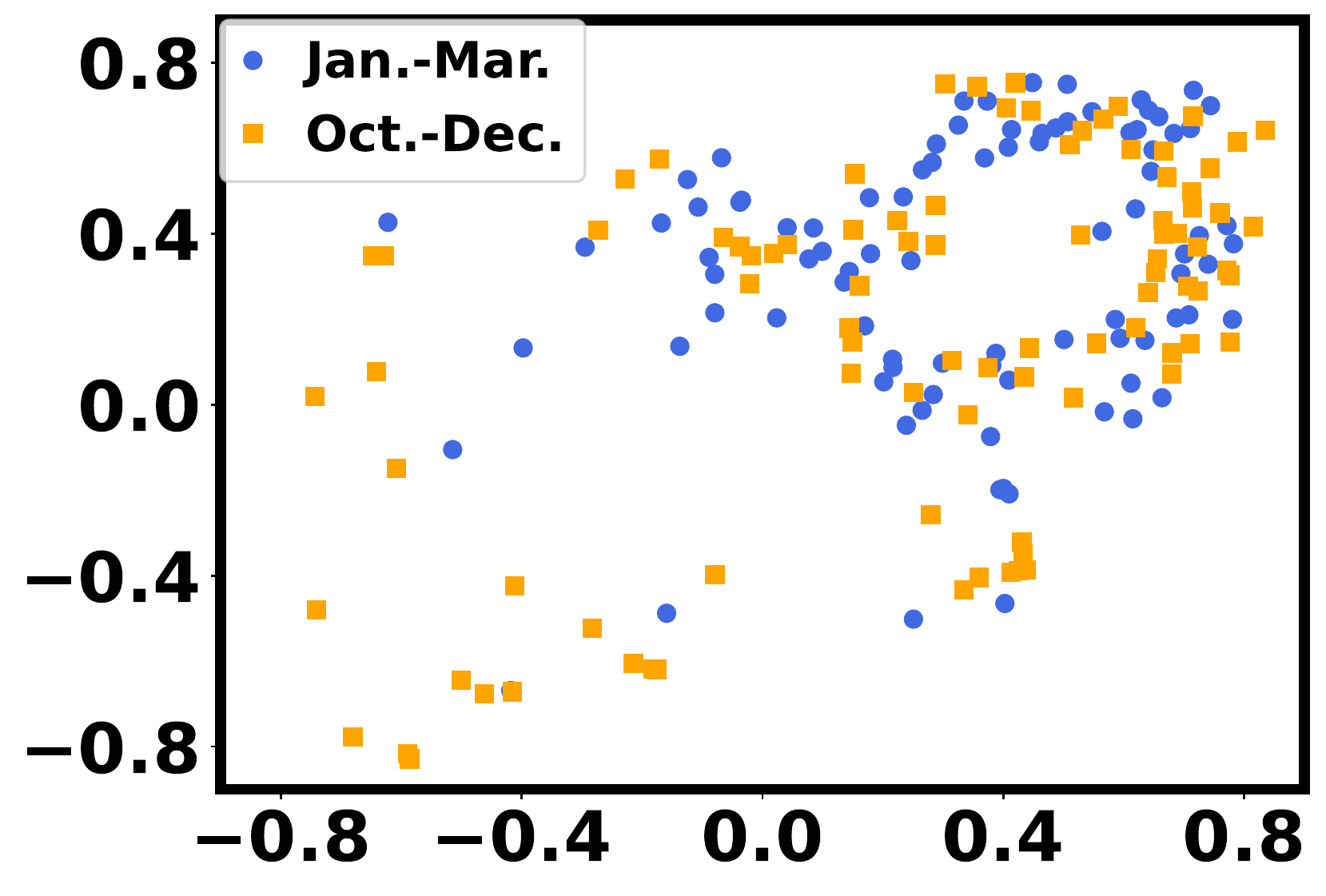}
	}
	\centering
	\vspace{-0.4cm}	
	\caption{t-SNE of distribution discrepancies between training months (January to March) and testing months.}
	\label{fig:ood}
\end{figure*}

In our recurrent structure, the ML model $h_W$ in each base optimizer includes 3 layers, each having 10 neurons.
We use $\mathrm{relu}$ activation and train the model weight using Adam optimization.

\subsection{Baseline Algorithms}

We compare \ouralg with several baselines including experts and ML-based algorithms. The baselines are summarized as below.

$\bullet$ Offline optimal oracle ($\textbf{Oracle}$): The offline optimal oracle
has access to the complete context information
in advance and produces the optimal solution.

$\bullet$ Regularized Online Balance Descent ($\textbf{R-OBD}$): R-OBD  is the
state-of-the-art expert algorithm that matches the lower bound of any online algorithms in terms of the competitive ratio \cite{SOCO_OBD_R-OBD_Goel_Adam_NIPS_2019_NEURIPS2019_9f36407e}. It uses the minimizer of the current hitting cost as a regularizer for online actions.
By setting $\lambda_3=0$, the expert calibrator \ouralgcal reduces to R-OBD.

$\bullet$ Pure ML-based Optimizer ($\textbf{PureML}$): PureML uses the same recurrent neural network as \ouralg, but it is trained
as a standalone optimizer without considering
the downstream expert calibrator. More specifically, we consider the loss function
$\kappa\frac{\mathrm{cost}(h_W,\bm{s})}{\mathrm{cost}(\pi^*,\bm{s})}+(1-\kappa)\mathrm{cost}(h_W,\bm{s})$ for the PureML model $h_W$, where $\kappa\in[0,1]$ balances the cost ratio loss and the average cost loss. Specifically, when $\kappa=0$, we denote the PureML model as PureML-0,
which is trained to minimize the average cost as in the standard L2O technique \cite{L2O_LearnWithoutGraident_ICML_2017_l2o_gradeint_descent_chen_ICML}.

$\bullet$ PureML with Dynamic Switching ($\textbf{Switch}$):
\revise{A switching-based  online algorithm
 dynamically switches between an individually
robust expert  and  ML predictions.
 The two most relevant switching-based online algorithms
 \cite{SOCO_OnlineMetric_UntrustedPrediction_ICML_2020_DBLP:conf/icml/AntoniadisCE0S20,SOCO_OnlineOpt_UntrustedPredictions_Switching_Adam_arXiv_2022}
consider switching costs in a metric space and hence are not directly applicable
for our squared switching costs.
Here, we modify
 Algorithm~1  in \cite{SOCO_OnlineMetric_UntrustedPrediction_ICML_2020_DBLP:conf/icml/AntoniadisCE0S20}
 by switching between R-OBD and pre-trained PureML
 with $\kappa=0$ (PureML-0)
 based on a parameterized threshold $\gamma$ that progressively increases itself for each occurrence of switch.
 This essentially follows the two-stage approach discussed in Section~\ref{sec:calibration_limitation_simple},
 and is simply referred to as Switch.
 Since the switching algorithm in \cite{SOCO_OnlineOpt_UntrustedPredictions_Switching_Adam_arXiv_2022}
  relies on triangle inequality of switching costs in a metric space,
  it is highly non-trivial to adapt the algorithm
  to our setting. Thus, we exclude it from our comparison.
  }

$\bullet$ PureML with \ouralgcal (\textbf{\ouralgcal}):
 To highlight the importance of training the ML-based optimizer $h_W$ by taking into account
the downstream calibrator,
we apply predictions of the pre-trained PureML-0 in \ouralgcal.
We simply refer to the whole algorithm as \ouralgcal.

Note that the all the costs  shown in our results are normalized by the average cost achieved by the offline optimal oracle.
In Figs.~\ref{fig:avg_cost} and~\ref{fig:cr},
the default hyperparameters
used for the algorithms under consideration
are $\kappa=0$ for PureML,
 $\gamma = 1.5$ for {Switch}, $\theta=0.3$ for {MLA-ROBD},
and $\theta=0.5$ and $\mu=0.6$ for \ouralg. The hyperparameters for R-OBD and $\lambda_2$ in \ouralgcal and \ouralg are optimally set according to Theorem~\ref{thm:cr_robd}.

\subsection{Results}

We present our results of the average cost and competitive ratio as follows. \revise{Note that the empirical regret result is omitted, because
it is already implicitly reflected by the average cost normalized with respect to the unconstrained optimal oracle's cost (e.g., a normalized average cost of 1.2 means that
the normalized average regret is 0.2).}

First, for the default hyperparameters, the normalized average costs and empirical competitive ratios of \ouralg and baselines are shown in Fig.~\ref{fig:avg_cost} and Fig.~\ref{fig:cr}, respectively. We can observe that the R-OBD achieves a very low competitive ratio, which is close
to the competitive ratio lower bound in \cite{SOCO_OBD_R-OBD_Goel_Adam_NIPS_2019_NEURIPS2019_9f36407e}. However, R-OBD has a very large average cost since it does not exploit any additional information such as statistical information of testing instances or ML predictions.

\begin{figure*}[!t]	
    \centering
\subfigure[Average cost]{\label{fig:avg_cost}
	\includegraphics[height=0.21\textwidth]{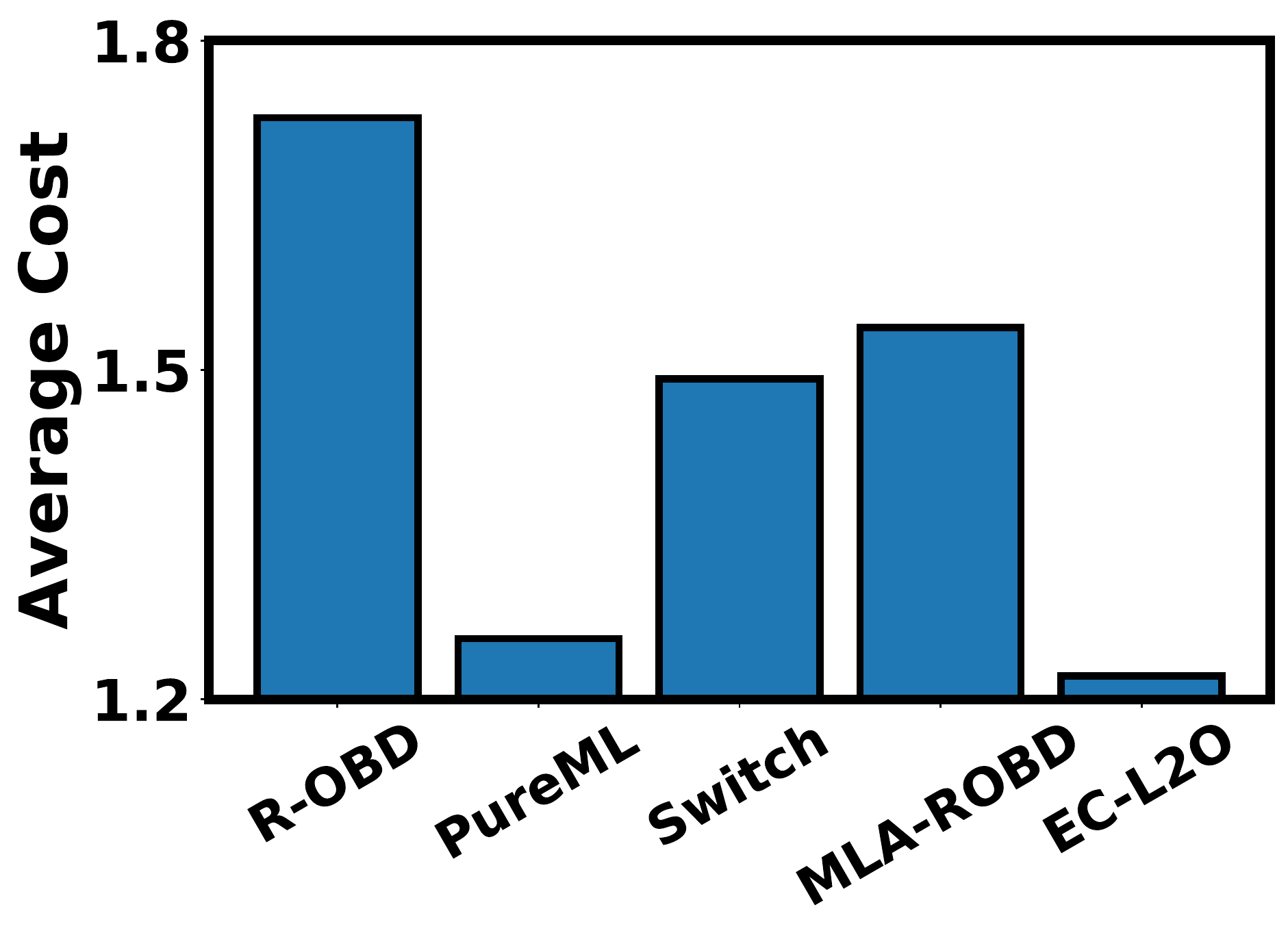}}	
\subfigure[Competitive ratio]{\label{fig:cr}
	\includegraphics[height=0.21\textwidth]{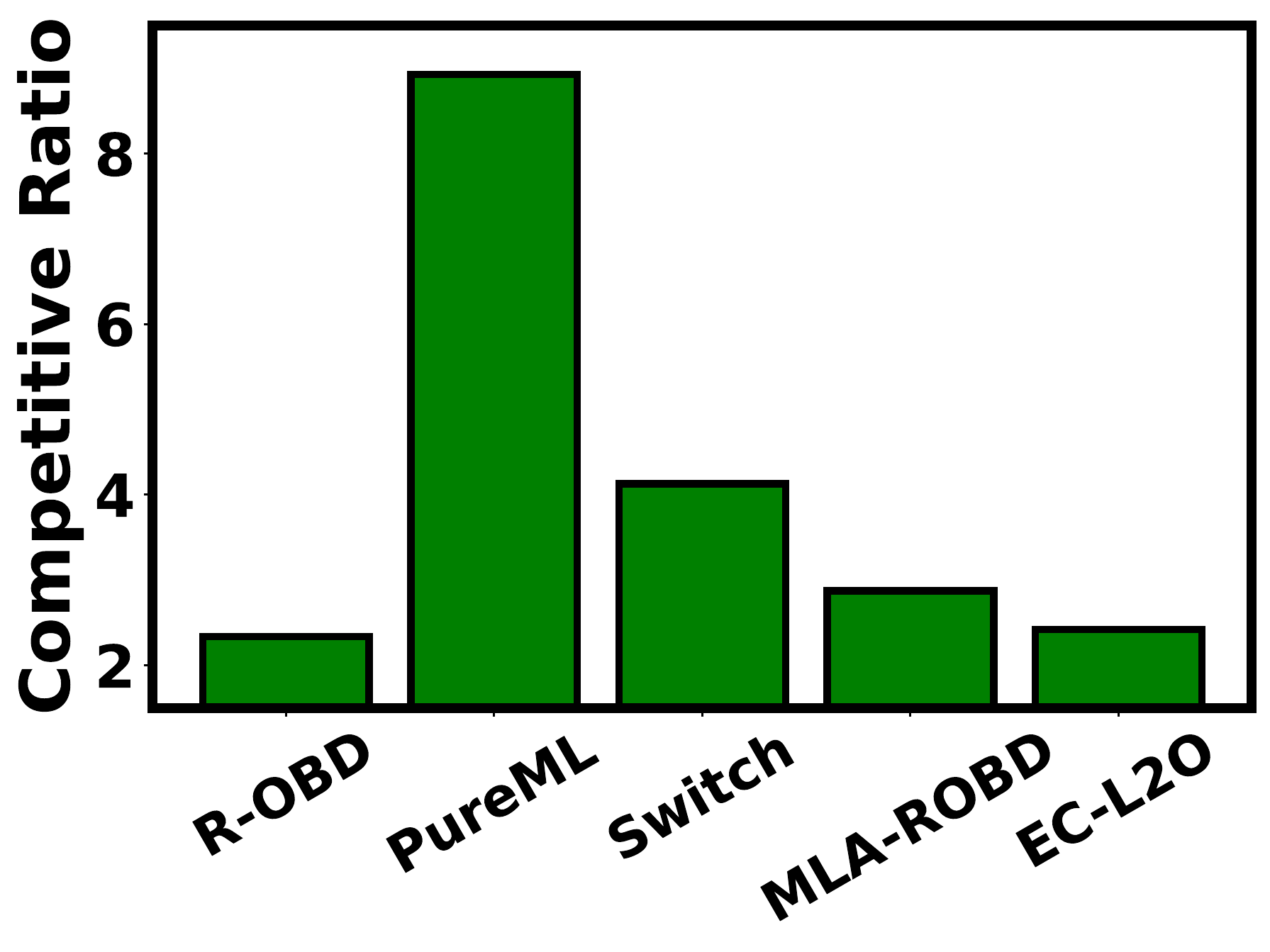}}
\subfigure[AVG vs. CR tradeoff]{
\label{fig:avgcost_cr}
	\includegraphics[height=0.21\textwidth]{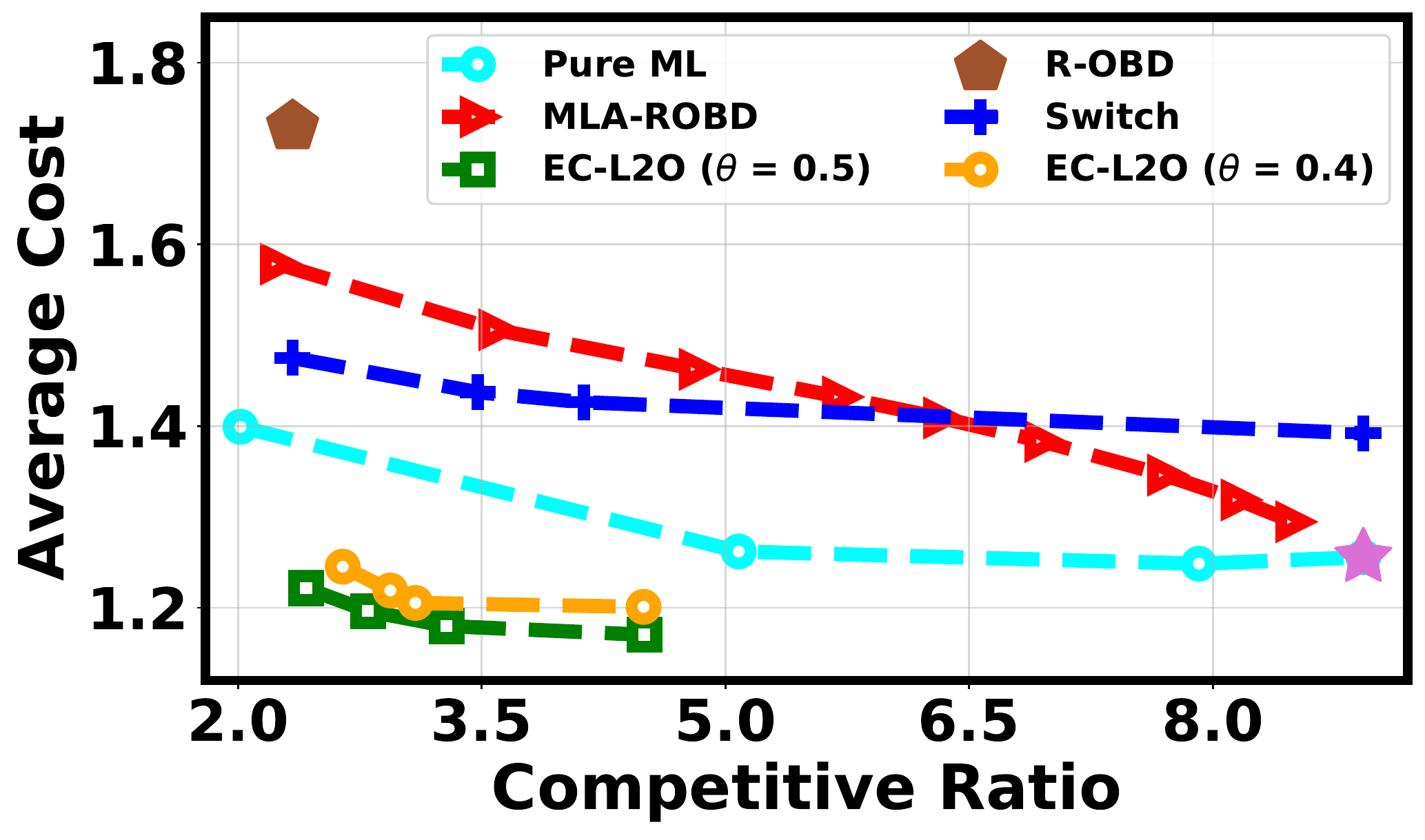}}
	\vspace{-0.4cm}
\caption{Performance comparison.}
\label{fig:compare_all}
\end{figure*}
By exploiting the input distributional information, PureML can reduce the average cost effectively, but due to the limitations of ML, it has unsatisfactory worst-case performance
--- an undesirably high competitive ratio. As ML-augmented online algorithms, Switch and \ouralgcal can effectively lower the competitive ratio, but  their average costs are much worse than PureML. The reasons are two-fold:
(1) the competitive ratios in these algorithms are often loose and hence
may not reflect the true empirical performance;
and (2)  the ML-based optimizer (PureML in this case) used by these algorithms is trained to minimize the pre-calibration cost using the standard practice of  L2O as a standalone optimizer without being aware of the downstream expert algorithm. This highlights that we cannot simply view the ML-based optimizer as an exogenous blackbox as in the existing ML-augmented algorithms \cite{SOCO_OnlineMetric_UntrustedPrediction_ICML_2020_DBLP:conf/icml/AntoniadisCE0S20}.

\revise{When the training and testing distributions are well consistent,
one may expect PureML to have the lowest average cost, if the ML model
has sufficient capacity and is well trained with a large number of training samples. Nonetheless,
for out-of-distribution testing in practice as shown in Fig.~\ref{fig:ood},
the number of \emph{hard} testing samples for PureML
is no longer non-negligible. As a consequence,
 PureML may not perform the best in terms of the average cost.
Interestingly, we see from Fig.~\ref{fig:compare_all}
that \ouralg can achieve an average cost even lower than PureML while keeping a competitive ratio that is not much higher than R-OBD.
This demonstrates the effectiveness of \ouralg in terms of reducing the average cost while restraining the empirical competitive ratio.
On the one hand, while the ML model in \ouralg is still vulnerable
to out-of-distribution samples, the downstream expert calibrator does not
depend on the training data and is
provably more robust. Thus, along
with our holistic training to minimize the post-calibration cost,
 \ouralg
results in a lower average cost than PureML
in the presence of training-testing distribution discrepancies.}
On the other hand,
although the competitive ratio of \ouralg is not theoretically upper bounded,
 the training loss function in Eqn.~\eqref{eqn:unrolling}
includes a loss of the ML prediction error which encourages the ML-based optimizer to have low prediction errors. Thus, by Theorem~\ref{thm:cr_robd},
this empirically reduces the competitive ratio achieved by \ouralg.

Next, in Fig.~\ref{fig:avgcost_cr}, we show the trade-off between
the average cost and competitive ratio by varying the respective hyperparameters that control the trade-off for different algorithms.  For each trade-off curve,
we only keep the Pareto boundary.
By setting the trust parameter to a low value of $\theta=0.2$, the leftmost point of \ouralgcal
can be even better than the pure R-OBD due to the introduction of ML predictions that are good in most cases. This is also reflected in Theorem~\ref{thm:cr_robd} and Fig.~\ref{fig:cr_bound}.
On the other hand, by $\theta=5$, ML predictions play a bigger role and \ouralgcal approaches
PureML with $\kappa=0$ (labeled as \emph{star}).
PureML can balance the average cost and the competitive ratio by adjusting the trade-off parameter $\kappa$. However, PureML only optimizes the trade-off performance based on an offline training dataset without an expert calibration layer, and hence it suffers from various limitations including generalization error, limited network capacity, error from testing distributional shifts.
Although Switch and \ouralgcal can achieve reasonably low competitive ratios, their average costs are higher than that of PureML and \ouralg.
On the other hand,
\ouralg with two different trust parameters $\theta=0.4$ and $\theta=0.5$ can achieve the empirically best Pareto front among all the algorithms: given an average cost, \ouralg has lower competitive ratios.
Also, we can find that if the trust parameter $\theta$ is higher, the  Pareto front becomes slightly better since more trust is given to the already-performant ML model.
Nonetheless, we cannot fully trust the ML predictions by setting $\theta\to\infty$, which would otherwise
significantly increase the competitive ratio.
\revise{From Fig.~\ref{fig:avgcost_cr}, we also see empirically
that the trade-off curve for \ouralg
is concentrated within a small region under different hyperparameters
(e.g., $\mu$
and $\theta=\frac{\lambda_3}{\lambda_1}$), showing that \ouralg
is not as sensitive to the hyperparameters  as other algorithms under comparison.}

Finally, we compare the tail cost ratios of algorithms at higher percentile greater than 99\% in Fig.~\ref{fig:tail_ratio}. We can observe that \ouralg achieves lower cost ratios than PureML, especially at percentile close to 100\%. Also, \ouralgpro has comparable  cost ratios with ML-augmented algorithms, i.e., Switch and \ouralgcal. This empirically confirms
 the conclusion in Theorem~\ref{thm:tail_ratio} that by  training on Eqn.~\eqref{eqn:loss_new} with $\mu>0$, \ouralg can constrain the cost ratio at high percentile, thus achieving a good tail robustness.

\begin{figure*}[!t]	
	\centering
	\subfigure{
		\label{fig:cdf_compare_1}
		\includegraphics[width=0.35\textwidth]{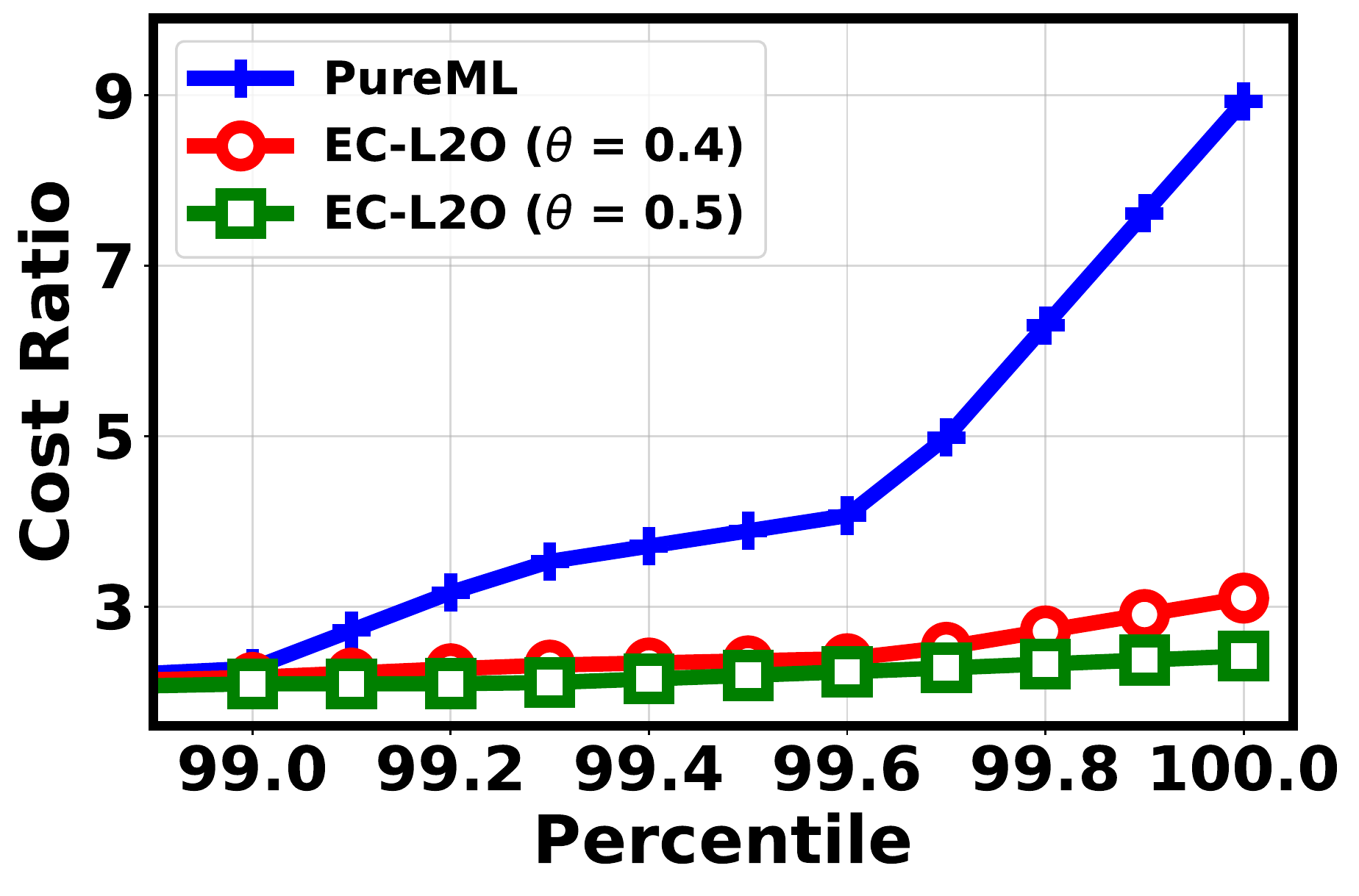}
	}
	\subfigure{
		\label{fig:cdf_compare_2}
		\includegraphics[width=0.35\textwidth]{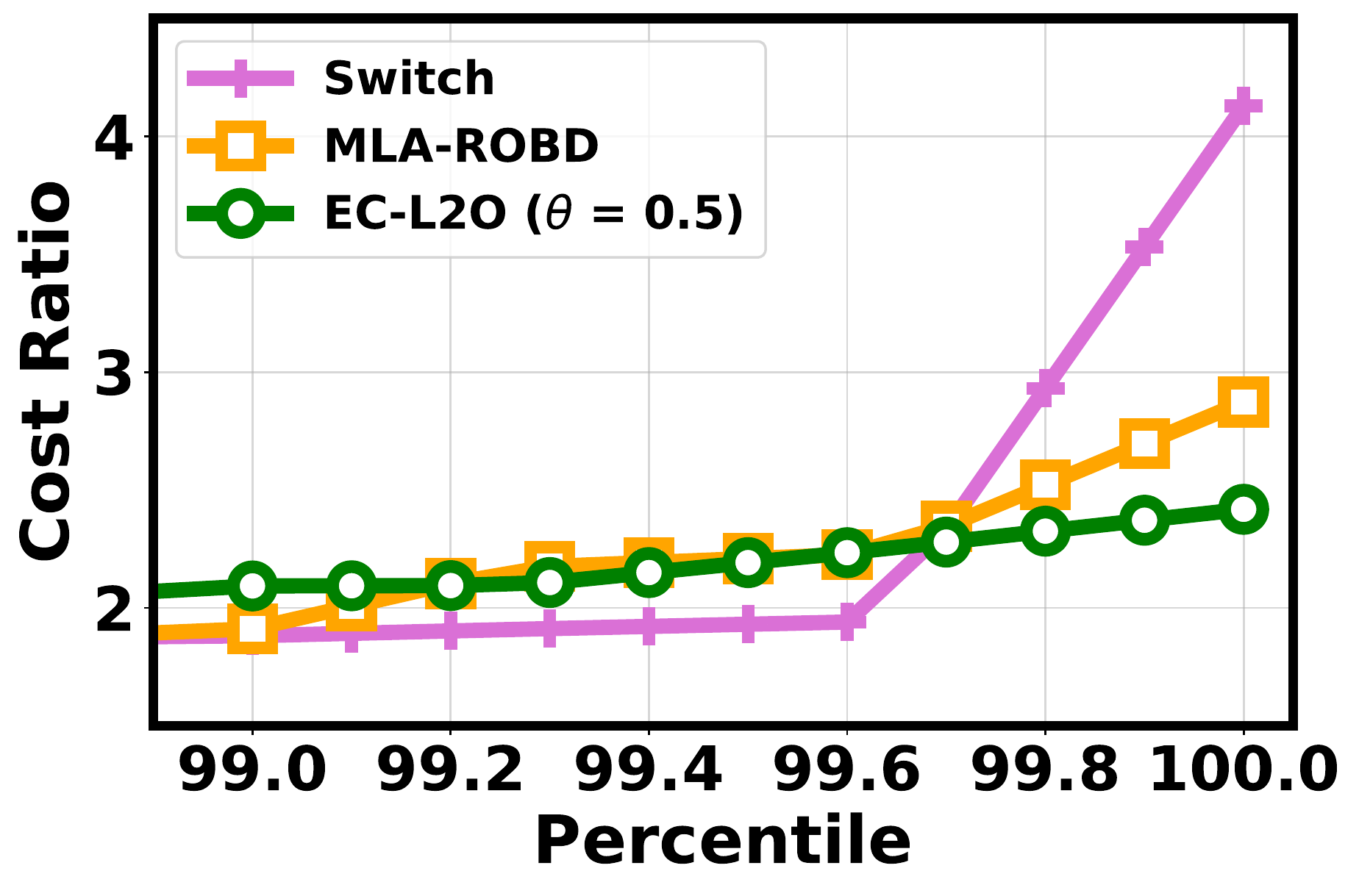}
	}
 \centering
 	\vspace{-0.4cm}	
	\caption{Tail cost ratio comparison between different algorithms}
	\label{fig:tail_ratio}
\end{figure*} 
\section{Related Works}

The set of expert-designed algorithms for online convex optimization with switching costs and other variant problems (e.g., convex boday chasing \cite{convex_body_chasing_friedman1993convex}
and metrical task system \cite{metrical_task_system_blum2000line}) have been
growing all the time \cite{SOCO_DynamicRightSizing_Adam_Infocom_2011_LinWiermanAndrewThereska,SOCO_Memory_Adam_NIPS_2020_NEURIPS2020_ed46558a,SOCO_OBD_Niangjun_Adam_COLT_2018_DBLP:conf/colt/ChenGW18,SOCO_Revisiting_Nanjing_NIPS_2021_zhang2021revisiting}.
For example, some prior studies have
considered
online gradient descent (OGD) \cite{OGD_zinkevich2003online,SOCO_Prediction_Error_Meta_ZhenhuaLiu_SIGMETRICS_2019_10.1145/3322205.3311087}, online balanced descent (OBD) \cite{SOCO_OBD_Niangjun_Adam_COLT_2018_DBLP:conf/colt/ChenGW18}, and regularized OBD (R-OBD) \cite{SOCO_OBD_R-OBD_Goel_Adam_NIPS_2019_NEURIPS2019_9f36407e}.
Typically, these algorithms are developed based on classic optimization frameworks and offer guaranteed worst-case performance robustness in terms of the competitive ratio.
 But, they may not have good cost performance in typical cases, thus potentially resulting
in a high average cost.

Additionally, some other studies have also
incorporated ML prediction
of future cost parameters  into the algorithm design under various settings.
Examples include
 receding horizon control (RHC) \cite{SOCO_Prediction_Error_Meta_ZhenhuaLiu_SIGMETRICS_2019_10.1145/3322205.3311087}
committed horizon control (CHC) \cite{SOCO_Prediction_Error_Niangjun_Sigmetrics_2016_10.1145/2964791.2901464},
receding horizon gradient descent (RHGD) \cite{Receding_Horizon_GD_li2020online,SOCO_Prediction_LinearSystem_NaLi_Harvard_NIPS_2019_10.5555/3454287.3455620},
and
adaptive balanced capacity scaling (ABCS) \cite{SOCO_CapacityScalingAdaptiveBalancedGradient_Gatech_MAMA_2021_Sigmetrics_2021_10.1145/3512798.3512808}. The goal of these algorithms is
still to achieve a low/bounded competitive ratio (or regret) in the presence of possibly large context prediction errors, while the average cost is left under-explored.

More recently, by combining ML-predicted actions with expert knowledge,
ML-augmented algorithm designs have been  emerging in the context of online convex optimization (or relevant problems such
as convex body/function chasing) with switching costs \cite{SOCO_OnlineMetric_UntrustedPrediction_ICML_2020_DBLP:conf/icml/AntoniadisCE0S20,SOCO_OnlineOpt_UntrustedPredictions_Switching_Adam_arXiv_2022,SOCO_ML_ChasingConvexBodiesFunction_Adam_UnderSubmission_2022}. They
focus on switching costs in a metric space while we consider squared switching costs.
 Moreover, the existing
studies focus on designing the expert calibration rule to achieve a low competitive ratio,
and hence take a rather simplified view
of the ML-based actions ---
the  actions  come from a black-box ML model. Thus, how to design
the ML-based optimizer that provides predictions to the downstream expert algorithm still remains open.
Simply following a two-stage approach in Section~\ref{sec:calibration_limitation_simple}
can provide unsatisfactory
 performance in terms of the average cost,  compared to the pure standalone ML-based optimizer that is trained for good average cost performance on its own.
Thus, in this paper, our proposal of \ouralg addresses a crucial yet unstudied challenge in
the emerging context of ML-augmented algorithm design: how to learn for online convex optimization with switching costs in the presence of a downstream expert?

 \ouralg also intersects with
 the quickly expanding area of
learning to optimize (L2O), which
pre-trains an ML-based optimizer to directly solve optimization problems
  \cite{L2O_LearningToOptimize_Berkeley_ICLR_2017,L2O_chen2021learning,L2O_generalize_wichrowska2017learned}.
Most commonly, L2O focuses on speeding up
the computation for
otherwise computationally expensive problems, such as
as DNN training \cite{Learng2learn_andrychowicz2016learning}, nonconvex optimization
in interference channels \cite{LearningOptimize_PowerAllocation_CongShen_TCOM_2020_8922744,Spatial_wireless_scheduling_cui2019spatial} and combination optimization \cite{Learn_combinatorial_opt_dai2017learning}.
 Moreover, ML-based optimizers have also been
integrated into
traditional algorithmic frameworks for
 faster and/or better solutions
\cite{L2O_chen2021learning,DNN_ImplicitLayers_Zico_Website}. But, these studies are dramatically different due to their orthogonal design goals and constraints.

Studies that apply L2O to solve difficult online optimization problems where the key challenge comes from the lack of complete offline information have been relatively
under-explored. In \cite{L2O_NewDog_OldTrick_Google_ICLR_2019},
an ML model is trained as a standalone end-to-end solution for a small set of classic online problems, such as
online knapsack. In \cite{L2O_AdversarialOnlineResource_ChuanWu_HKU_TOMPECS_2021_10.1145/3494526}
proposes adversarial training based on generative adversarial networks to solve online resource allocation problems.
But, such an end-to-end ML-based optimizer can have bad performance without provably good robustness. While feeding
the predictions of a pre-trained ML model directly into an expert calibrator can improve the robustness, the average cost performance can be significantly damaged due to the siloed two-stage approach.

In the recent ``predict-then-optimize'' and decision-focused learning framework
\cite{L2O_PredictOptimize_MDP_Harvard_NIPS_2021_wang2021learning,L2O_PredictOptimize_RiskCalibrationBound_ICLR_2021_liu2021risk,decision_focused_learning_combinatorial_opt_wilder2019melding,L2O_PredictOptimize_arXiv_2017},
auxiliary contextual information is predicted based on additional input features by taking into account the downstream optimizer in order to minimize the overall decision cost.
A key research problem
is to differentiate the optimizer with respect to the predicted
contextual information in order to facilitate the backpropagation process for efficient
training \cite{L2O_DifferentiableConvexOptimization_Brandon_NEURIPS2019_9ce3c52f,L2O_DifferentiableMPC_NIPS_2018_NEURIPS2018_ba6d843e,L2O_DifferentiableOptimization_Brandon_ICML_2017_amos2017optnet}.
Nonetheless, these studies typically focus on the prediction
step while considering an existing optimizer with the goal of optimizing the average performance,
whereas we consider multi-step online optimization
 and propose a new expert calibrator to optimize the average
 performance with provably improved robustness.
Additionally,  although the optimizer produces actions based on the predicted
contextual information, the ``predict-then-optimize'' framework
still uses the true contextual information to evaluate
the cost; by contrast, we directly utilize ML together with
an expert calibrator to produce
online actions and evaluate the total cost, and the notion of ``true''
 actions does not apply in our problem.

Finally, a small set of recent studies
\cite{L2O_CustomizingML_OnlineAlgorithm_GeRong_Duke_ICML_2020,L2O_LearningMLAugmented_Regression_CR_GeRong_NIPS_2021_anand2021a,L2O_LearningMLAugmentedAlgorithm_Harvard_ICML_2021_pmlr-v139-du21d} have begun to study ``how to learn'' in ML-augmented algorithm designs for different goals. For example, \cite{L2O_LearningMLAugmentedAlgorithm_Harvard_ICML_2021_pmlr-v139-du21d}
designs an alternative ML model
for the count-min sketch problem,
while \cite{L2O_LearningMLAugmented_Regression_CR_GeRong_NIPS_2021_anand2021a}
considers general online problems and
uses ML (i.e., regression) to minimize
the \emph{average} competitive ratio without considering the average cost.
Our work is novel in multiple aspects, including the problem setting, expert calibrator, loss function design, backpropagation process
and performance analysis.

\section{Conclusion}

In this paper, we study online convex optimization with switching costs.
We show that
by using the standard
practice of training an ML model
as a standalone optimizer,
 ML-augmented online algorithms
 can significantly hurt the average cost performance.
Thus, we propose
\ouralg, which
trains
an ML-based optimizer by explicitly taking into account the downstream expert calibrator.
 We design a new differentiable
expert calibrator, \ouralgcal, that generalizes R-OBD and
offers a provably better
competitive ratio than pure ML predictions when the prediction error is large.
 We also provide theoretical analysis for \ouralg, highlighting
that the
high-percentile tail cost ratio can be bounded.
Finally, we test \ouralg by running simulations for  sustainable datacenter demand response,
showing that
that \ouralg can empirically achieve a lower average cost as well
as a lower competitive ratio than the existing baseline algorithms.

\bibliographystyle{ACM-Reference-Format}



\newpage
\appendix

\section*{Appendix}

\section{Additional Details of Differentiating the Calibrator}\label{sec:gradientproof}
We now differentiate the output $x_t$ of the calibrator \ouralgcal with respect to its inputs $\tilde{x}_t$ and $x_{t-1}$. By the optimum condition of the convex optimization to calculate  $x_t$ in Algorithm \ref{alg:RP-OBD}, we have
\begin{equation}\label{eqn:gradientproof}
\nabla_{x_t} f(x_t, y_t) + \lambda_1 \nabla_{x_t} c(x_t, x_{t-1}) + \lambda_2 \nabla_{x_t} c(x_t, v_t) + \lambda_3 \nabla_{x_t} c(x_t, \tilde{x}_t) = 0.
\end{equation}
To derive the gradient of $x_t=R_{\lambda}(y_t,x_{t-1},\tilde{x}_t)$ with respect to $\tilde{x}_t$, we take gradient for both sides of Eqn.~\eqref{eqn:gradientproof}, and get
\begin{equation}
\begin{split}
    &\left(\nabla_{x_t,x_t} f(x_t, y_t) + \lambda_1 \nabla_{x_t,x_t} c(x_t, x_{t-1}) + \lambda_2 \nabla_{x_t,x_t} c(x_t, v_t) + \lambda_3 \nabla_{x_t,x_t} c(x_t, \tilde{x}_t)\right)\nabla_{\tilde{x}_t}R_{\lambda}(y_t,x_{t-1},\tilde{x}_t)\\
    &+ \lambda_3 \nabla_{\tilde{x}_t,x_t} c(x_t, \tilde{x}_t)= 0,
\end{split}
\end{equation}
where $\nabla_{x_b,x_a}$ means first taking gradient with respect to $x_a$, and then taking gradient with respect to $x_b$. Denote $Z_t=\nabla_{x_t,x_t} f(x_t, y_t) + \lambda_1 \nabla_{x_t,x_t} c(x_t, x_{t-1}) + \lambda_2 \nabla_{x_t,x_t} c(x_t, v_t) + \lambda_3 \nabla_{x_t,x_t} c(x_t, \tilde{x}_t)$. Due to the assumption that $f(x_t, y_t)$ and $c$ are strongly with respect to $x_t$,  we have $Z_t$ is positive definite and invertible and so
\begin{equation}
    \nabla_{\tilde{x}_t}R_{\lambda}(y_t,x_{t-1},\tilde{x}_t)=-\lambda_3 Z_t^{-1}\nabla_{\tilde{x}_t,x_t} c(x_t, \tilde{x}_t).
\end{equation}
Similarly to derive the gradient of $x_t=R_{\lambda}(y_t,x_{t-1},\tilde{x}_t)$ with respect to $x_{t-1}$, we take gradient for both sides of Eqn.~\eqref{eqn:gradientproof}, and get
\begin{equation}
\begin{split}
    &\left(\nabla_{x_t,x_t} f(x_t, y_t) + \lambda_1 \nabla_{x_t,x_t} c(x_t, x_{t-1}) + \lambda_2 \nabla_{x_t,x_t} c(x_t, v_t) + \lambda_3 \nabla_{x_t,x_t} c(x_t, \tilde{x}_t)\right)\nabla_{x_{t-1}}R_{\lambda}(y_t,x_{t-1},\tilde{x}_t)\\
    &+ \lambda_1 \nabla_{x_{t-1},x_t} c(x_t, x_{t-1})= 0.
\end{split}
\end{equation}
Thus, we have
\begin{equation}
    \nabla_{x_{t-1}}R_{\lambda}(y_t,x_{t-1},\tilde{x}_t)=-\lambda_1 Z_t^{-1}\nabla_{x_{t-1},x_t} c(x_t, x_{t-1}).
\end{equation}

\section{Proof of Lemma~\ref{thm:ml_eta}}\label{sec:proof_ml_eta}
\begin{proof}
Denote $\{\tilde{x}_1,\cdots,\tilde{x}_T\}$ as the $\rho$-accurate predictions of an ML model $h_{W}$. Let the output of the offline optimal policy  $\pi^*$ be $\{x_1^*, x_2^*, \cdots , x_T^*\}$.
If the competitive ratio of the ML model is $\eta$, then by the definition of competitive ratio, we have $\mathrm{cost}(h_{W}, \bm{s}) \leq \eta \mathrm{cost}(\pi^*, \bm{s})$, for any $\bm{s}\in\mathcal{S}$ such that the ML prediction is $\rho$-accurate. Thus, we prove the lower bound of the competitive ratio of the $\rho$-accurate ML model by finding a sequence $\bm{s}$ with predicted action sequence such that $\sum_{t=1}^T\|\tilde{x}_t- x^*_t\| \leq \rho \mathrm{cost}(\pi^*, \bm{s} )$ and the cost ratio $\frac{\mathrm{cost}(h_W, \bm{s})}{\mathrm{cost}(\pi^*, \bm{s})} \geq 1+\frac{m+2\alpha}{2}$.

 We find the sequence of ML predictions by replacing the first action of the offline-optimal action sequence with $\hat{x}_1\in\mathcal{X}, ||\hat{x}_1 - x_1^*||^2= \rho cost(\pi^*, \bm{s} )$, i.e.  $\{\tilde{x}_1,\cdots,\tilde{x}_T\}=\{\hat{x}_1, x_2^*, \cdots , x_T^*\}$. We then prove that the cost ratio of this ML prediction sequence is at least $1+\frac{m+2\alpha}{2}$.

 Since  $\{x_1^*, x_2^*, \cdots , x_T^*\}$ is the offline optimal sequence that minimizes Eqn.~\eqref{eqn:obj}, then $x_1^*$ must be the minimizer of the function $p(x)= f(x, y_1) +  c(x, x_0) +  c(x, x_2^*)$. Since the hitting cost $f(x,y)$ is $m$-strongly convex in terms of $x$, and the switching cost $c(x_a,x_b)$ is $\alpha$-strongly convex in terms of $x_a$, we have
\begin{equation}\label{eqn:firstordergradient}
\nabla p(x_1^*)  = \nabla_{x_1^*} f(x_1^*, y_1) + \nabla_{x_1^*} c(x_1^*, x_0) +  \nabla_{x_1^*} c(x_1^*, x_2^*) = 0.
\end{equation}
Thus we have
\begin{equation}
		\begin{split}
    &\mathrm{cost}(h_W, \bm{s}) - \mathrm{cost}(\pi^*, \bm{s})\\
    =& p(\hat{x}_1)-p(x^*_1)\\
    =&f(\hat{x}_1, y_1) - f(x_1^*, y_1) + c(x_0, \hat{x}_1) - c(x_0, x_1^*) + c(\hat{x}_1, x_2^*) - c(x_1^*, x_2^*)\\
    \geq &  (\nabla_{x_1^*} f(x_1^*, y_1) + \nabla_{x_1^*} c(x_1^*, x_0) +  \nabla_{x_1^*} c(x_1^*, x_2^*))^T \cdot (\hat{x}_1 - x_1^*) + (\frac{m}{2} + \frac{\alpha}{2} + \frac{\alpha}{2}) ||\hat{x}_1 - x_1^*||^2  \\
    =& \frac{m + 2 \alpha }{2} ||\hat{x}_1 - x_1^*||^2\\
    =&\frac{m + 2 \alpha }{2} \rho \mathrm{cost}(\pi^*, \bm{s} ),
\end{split}
\end{equation}
where the inequality holds by the $(m+2\alpha)-$strong convexity of $p$, and the third equality comes from Eqn.~\eqref{eqn:firstordergradient}. Therefore, the cost ratio for the selected prediction sequence $\frac{cost(h_W, \bm{s})}{cost(\pi^*, \bm{s})}$ is at least $1+\frac{m+2\alpha}{2}\rho$, and this completes the proof.
\end{proof}

\section{Proof of Theorem~\ref{thm:cr_robd}}\label{appendix:proof_cr_robd}
\begin{proof}

The switching cost $c(x_t,x_{t-1})$ is measured
 in terms of the squared Mahalanobis distance with respect to a symmetric and positive-definite  matrix $Q\in\mathcal{R}^{d\times d}$, i.e. $c(x_t,x_{t-1})=\left(x_t-x_{t-1}\right) ^\top Q \left(  x_t-x_{t-1}\right) $ \cite{mahalanobis_distance_de2000mahalanobis}.  Since
($\frac{\alpha}{2}, \frac{\beta}{2}$)  are the smallest and largest eigenvalues of matrix $Q$, then we have
\begin{equation}
    \label{eq:BD_inequal}
    \frac{\alpha}{2} ||x - y||^2 \leq c(x,y) \leq \frac{\beta}{2} ||x-y||^2, \quad \forall x,y \in R^d.
\end{equation}

Since $x_t$ is the solution of the Line~4 of Algorithm~\ref{alg:RP-OBD}, we have
\begin{equation}
    \label{eq:alg_grad}
    \begin{split}
    \nabla f(x_t, y_t) = &-2\lambda_1Q( x_t- x_{t-1}) - 2\lambda_2Q(x_t - v_t) -2\lambda_3Q(x_t - \tilde{x}_t).
    \end{split}
\end{equation}

Since $f(x,y_t)$ is $m$-strongly convex in terms of $x$, we have
\begin{equation}
\label{eq:alg_grad_ineq}
f(x_t^*, y_t) \geq f(x_t, y_t) + \langle \nabla f(x_t, y_t), x_t^* - x_t  \rangle + \frac{m}{2} ||x_t^* - x_t||^2.
\end{equation}

By substituting \eqref{eq:alg_grad} into \eqref{eq:alg_grad_ineq}, we have
\begin{equation}
\label{eq:alg_grad_ineq2}
\begin{split}
f(x_t^*, y_t) \geq& f(x_t, y_t) -2\lambda_1(x_t^*-x_t)Q(x_t-x_{t-1})-2\lambda_2(x_t^*-x_t)Q(x_t-v_t)\\
&-2\lambda_3(x_t^*-x_t)Q(x_t-\tilde{x}_t)+ \frac{m}{2} ||x_t^* - x_t||^2.
\end{split}
\end{equation}

For the Mahalanobis distance,  the following property holds for any $x,y,z\in\mathcal{R}^d$
\begin{equation}
\label{eq:BD_grad}
 c(x,y) - c(x,z) - c(z,y)=2(y-z)^T Q (z-x).
\end{equation}
By Eqn.~\eqref{eq:BD_grad} and moving the terms in \eqref{eq:alg_grad_ineq2}, we have
\begin{equation}
    \label{eq:complex_grad}
    \begin{split}
    &f(x_t^*, y_t) + \lambda_1 c(x_t^*,x_{t-1}) + \lambda_2 c(x_t^*, v_t) + \lambda_3 c(x_t^*, \tilde{x}_t)\\
     \geq& f(x_t, y_t) + \lambda_1 c(x_t,x_{t-1}) + \lambda_2 c(x_t, v_t) + \lambda_3 c(x_t , \tilde{x}_t)
     + (\lambda_1+\lambda_2+\lambda_3)c(x_t,x_t^*) + \frac{m}{2} \|x_t^* - x_t \|^2\\
     \geq& f(x_t, y_t) + \lambda_1 c(x_t,x_{t-1})
     + (\lambda_1+\lambda_2+\lambda_3)c(x_t,x_t^*) + \frac{m}{2} \|x_t^* - x_t \|^2
    \end{split}
\end{equation}
where the last inequality holds since the Mahalanobis distance is non-negative.

We define a function
$ \phi(x_t, x_t^*) = (\lambda_1 + \lambda_2 + \lambda_3) c(x_t, x_t^*) + \frac{m}{2} ||x_t^* - x_t ||^2 $, and let $\Delta \phi_t = \phi(x_t, x_t^*) - \phi(x_{t-1}, x_{t-1}^*)$.
Then subtracting $\phi(x_{t-1}, x_{t-1}^*)$ from both sides of \eqref{eq:complex_grad}, we have
\begin{equation}
    \label{eq:l_CR_part}
    \begin{split}
     &f(x_t,y_t) + \lambda_1 c(x_t,x_{t-1}) + \Delta \phi_t\\
    \leq & f(x_t^*, y_t) + \lambda_2 c(x_t^*, v_t) + \lambda_1 c(x_t^*,x_{t-1}) - (\lambda_1 + \lambda_2 + \lambda_3) c(x_{t-1}, x_{t-1}^*) \\
    &  - \frac{m}{2} ||x_{t-1}^* - x_{t-1} ||^2 + \lambda_3 c(x_t^*, \tilde{x}_t).
    \end{split}
\end{equation}

Next,
we can prove that
\begin{equation}
    \label{eq:BD_part}
    \begin{split}
     &  \lambda_1 c(x_t^*,x_{t-1}) - (\lambda_1 + \lambda_2 + \lambda_3) c(x_{t-1} , x_{t-1}^*) - \frac{m}{2} ||x_{t-1}^* - x_{t-1} ||^2 \\
    = &  \lambda_1 c(x_t^*,x_{t-1}^*) +  2\lambda_1
    (x_t^* - x_{t-1}^*)^T Q (x_{t-1}^* - x_{t-1}) - (\lambda_2 + \lambda_3) c(x_{t-1} , x_{t-1}^*) - \frac{m}{2} ||x_{t-1}^* - x_{t-1} ||^2 \\
    \leq &  \lambda_1 c(x_t^*,x_{t-1}^*) + \lambda_1 \beta ||x_t^*- x_{t-1}^*|| \cdot ||x_{t-1}^* - x_{t-1}|| - \frac{ \beta(\lambda_2 + \lambda_3) +m}{2} ||x_{t-1}^* - x_{t-1} ||^2\\
    \leq &  \lambda_1 c(x_t^*,x_{t-1}^*) + \frac{ \lambda_1^2 \beta^2}{2(\beta(\lambda_2 + \lambda_3) + m)}||x_t^* - x_{t-1}^* ||^2 \\
    \leq &  \lambda_1 c(x_t^*,x_{t-1}^*) + \frac{ \lambda_1^2 \beta^2}{\alpha(\beta(\lambda_2 + \lambda_3) + m)} c(x_t^*,x_{t-1}^*) \\
    = &  \lambda_1(1  + \frac{\lambda_1\beta^2}{\alpha ((\lambda_2+
    \lambda_3)\beta+m)}) c(x_t^*,x_{t-1}^*),
    \end{split}
\end{equation}
where the first equality holds by Eqn.~\eqref{eq:BD_grad}, the first inequality comes from the assumption that matrix $Q$ has the largest eigenvalue $\frac{\beta}{2}$, the second inequality holds by the inequality of arithmetic and geometric means (AM-GM inequality) such that $\lambda_1\beta ||x_t^*- x_{t-1}^*|| \cdot ||x_{t-1}^* - x_{t-1}||=\sqrt{\frac{\lambda_1^2\beta^2}{\beta(\lambda_2 + \lambda_3) +m} ||x_t^*- x_{t-1}^*||^2 \cdot (\beta(\lambda_2 + \lambda_3) +m)||x_{t-1}^* - x_{t-1}||^2}\leq \frac{ \lambda_1^2 \beta^2}{2(\beta(\lambda_2 + \lambda_3) + m)}||x_t^* - x_{t-1}^* ||^2+\\\frac{ \beta(\lambda_2 + \lambda_3) +m}{2} ||x_{t-1}^* - x_{t-1} ||^2$, and the third inequality holds because the switching cost function is $\alpha$-strongly convex.

Following \eqref{eq:l_CR_part}, we have
\begin{equation} \label{eqn:proofcr1}
\begin{split}
&f(x_t,y_t) + \lambda_1 c(x_t,x_{t-1}) + \Delta \phi_t \\
\leq &f(x_t^*, y_t) + \lambda_2 c(x_t^*, v_t) +  \lambda_1(1  + \frac{\lambda_1\beta^2}{\alpha ((\lambda_2+
    \lambda_3)\beta+m)})c(x_t^*,x_{t-1}^*) + \lambda_3 c(x_t^*,\tilde{x}_t) \\
    \leq & f(x_t^*, y_t) +  \lambda_1(1  + \frac{\lambda_1\beta^2}{\alpha ((\lambda_2+
    \lambda_3)\beta+m)})c(x_t^*,x_{t-1}^*)+ \frac{\lambda_2 \beta}{2} ||x_t^* - v_t ||^2  + \frac{\lambda_3 \beta}{2} ||x_t^* - \tilde{x}_t ||^2,
    \end{split}
    \end{equation}
where the last inequality comes from inequality \eqref{eq:BD_inequal}.

Then, summing up the inequalities of \eqref{eqn:proofcr1} for $t=1,\cdots, T$ and by the fact that $x_0=x_0^*$, we have
\begin{equation} \label{eqn:proofcr2}
\begin{split}
&\sum_{t=1}^Tf(x_t,y_t) + \lambda_1 \sum_{t=1}^Tc(x_t,x_{t-1}) + \phi(x_T,x_T^*) \\
    \leq & \sum_{t=1}^Tf(x_t^*, y_t) +  \lambda_1(1  + \frac{\lambda_1\beta^2}{\alpha ((\lambda_2+
    \lambda_3)\beta+m)})\sum_{t=1}^Tc(x_t^*,x_{t-1}^*)+ \frac{\lambda_2 \beta}{2} \sum_{t=1}^T||x_t^* - v_t ||^2  + \frac{\lambda_3 \beta}{2} \sum_{t=1}^T||x_t^* - \tilde{x}_t ||^2\\
    \leq &\frac{m+\lambda_2\beta}{m}\sum_{t=1}^Tf(x_t^*, y_t) +  \lambda_1(1  + \frac{\lambda_1\beta^2}{\alpha ((\lambda_2+
    \lambda_3)\beta+m)})\sum_{t=1}^Tc(x_t^*,x_{t-1}^*)  + \frac{\lambda_3 \beta}{2} \sum_{t=1}^T||x_t^* - \tilde{x}_t ||^2,
    \end{split}
    \end{equation}
    where the second inequality holds by $m$-strong convexity of $f$  and the fact that $v_t$ is the minimizer of $f(x,y_t)$ such that $f(x_t^*,y_t)\geq f(v_t,y_t)+\frac{m}{2}\|x_t^*-v_t\|$,  and so $ \frac{1}{2}\|x_t^*-v_t\|\leq \frac{1}{m}f(x_t^*,y_t)$ since $f(v_t,y_t)\geq 0$.
Since $0<\lambda_1\leq 1$, we have
\begin{equation} \label{eqn:proofcr3}
\begin{split}
&\sum_{t=1}^Tf(x_t,y_t) +  \sum_{t=1}^Tc(x_t,x_{t-1})\leq \sum_{t=1}^T\frac{1}{\lambda_1}f(x_t,y_t) +  \sum_{t=1}^Tc(x_t,x_{t-1}) + \frac{1}{\lambda_1}\phi(x_T,x_T^*) \\
    \leq & \frac{m+\lambda_2\beta}{m\lambda_1}\sum_{t=1}^Tf(x_t^*, y_t) +  (1  + \frac{\lambda_1\beta^2}{\alpha ((\lambda_2+
    \lambda_3)\beta+m)})\sum_{t=1}^Tc(x_t^*,x_{t-1}^*)  + \frac{\lambda_3 \beta}{2\lambda_1} \sum_{t=1}^T||x_t^* - \tilde{x}_t ||^2\\
    \leq & \frac{m+\lambda_2\beta}{m\lambda_1}\sum_{t=1}^Tf(x_t^*, y_t) +  (1  + \frac{\beta^2}{\alpha}\frac{\lambda_1}{ ((\lambda_2+
    \lambda_3)\beta+m)})\sum_{t=1}^Tc(x_t^*,x_{t-1}^*)  + \frac{\lambda_3 \beta}{2\lambda_1}\rho \sum_{t=1}^T\mathrm{cost}(\pi^*,\bm{s})\\
    \leq & \max\left(\frac{m+\lambda_2\beta}{m\lambda_1}, 1  + \frac{\beta^2}{\alpha}\frac{\lambda_1}{ ((\lambda_2+
    \lambda_3)\beta+m)}\right)\mathrm{cost}(\pi^*,\bm{s})  + \frac{\lambda_3 \beta}{2\lambda_1}\rho \sum_{t=1}^T\mathrm{cost}(\pi^*,\bm{s})
    \end{split}
    \end{equation}
    where the second inequality holds because the ML prediction sequence $\{\tilde{x}_1,\cdots,\tilde{x}_T\}$ is $\rho$-accurate.
   Therefore, the competitive ratio of \ouralg is bounded by
$\max (\frac{m+ \lambda_2 \beta }{m\lambda_1}, 1+\frac{\beta^2}{\alpha} \cdot \frac{\lambda_1}{(\lambda_2+\lambda_3)\beta + m}) + \frac{\lambda_3 \beta}{2\lambda_1}\rho$.

Next we prove the optimal setting for the selection of parameters.
Given any $\lambda_1$ and $\lambda_3$, in order to minimize the competitive ratio in Theorem \ref{thm:cr_robd}, $\lambda_2$ must satisfy
$$\frac{m+ \lambda_2 \beta }{m\lambda_1} =  1+\frac{\beta^2}{\alpha} \cdot \frac{\lambda_1}{(\lambda_2+\lambda_3)\beta + m}.$$
This function can be restated as
$$\frac{m+ (\lambda_2+\lambda_3) \beta }{m\lambda_1} - 1 - \frac{\lambda_3 \beta}{m \lambda_1}=  \frac{\beta^2}{\alpha} \cdot \frac{\lambda_1}{(\lambda_2+\lambda_3)\beta + m}.$$

If we substitute $\frac{m+ (\lambda_2+\lambda_3) \beta }{\lambda_1}$ as $k_1$, we have
$$\frac{k_1}{m} -(1 + \frac{\lambda_3 \beta}{m \lambda_1})  = \frac{\beta^2}{\alpha} \frac{1}{k_1}.$$

Since $k_1>0$, the root for this equation is
$$k_1 = \frac{m}{2}(1+\frac{\lambda_3 \beta}{m \lambda_1} + \sqrt{(1+\frac{\lambda_3 \beta }{m \lambda_1})^2 + \frac{4 \beta^2}{m \alpha}} ).$$

If we substitute $k_1$ back to $\lambda_2$, since $\lambda_2 = \frac{\lambda_1}{\beta} (k_1 - \frac{m + \lambda_3 \beta }{\lambda_1})$, then
\begin{equation}\label{eqn:proofcondition}
\lambda_2=\frac{m\lambda_1}{2\beta}\left(\sqrt{\left(1+\frac{\beta\lambda_3}{m\lambda_1}\right)^2+\frac{4\beta^2}{\alpha m}}+1-\frac{2}{\lambda_1}-\frac{\beta\lambda_3}{m \lambda_1}\right).
\end{equation}

If $\lambda_1 = 1$, it is obvious that the solution for $\lambda_2$ is positive, which satisfy the condition of $\lambda_2 \geq 0$.
If $\lambda_1 < 1$, then $\lambda_2$ decreases monotonically as $\lambda_1$.
Suppose when $\lambda_2=0$, $\lambda_1^{min}$ should be selected as $\lambda_1$ to meet Eqn.~\eqref{eqn:proofcondition}.
If $1 \geq \lambda_1 \geq \lambda_1^{min}$, the solution to $\lambda_2$ is non-negative, then the competitive ratio becomes
$$1 + \frac{2\beta^2}{m\alpha (\sqrt{(1+\frac{\lambda_3 \beta }{m \lambda_1})^2 + \frac{4 \beta^2}{m \alpha}} + 1 + \frac{\lambda_3 \beta}{m \lambda_1})} + \frac{\lambda_3 \beta}{2 \lambda_1}\rho.$$
From this equation we can find that, the optimal competitive ratio is only affected by the trust parameter $\theta=\frac{\lambda_3}{\lambda_1}$ if the parameter setting in Eqn.\eqref{eqn:proofcondition} is satisfied.
\end{proof}

\section{Proof of Theorem~\ref{thm:regret}}\label{appendix:proof_regret_mla}
For the convenience of presentation, we use $x_{1:T}$ to denote a sequence of actions $x_1,\cdots,x_T$,
and $cost(x_{1:T})$ to denote the corresponding total
cost $\sum_{t=1}^T f(x_t)  + c(x_{t-1}, x_t)$.
Next, we show the following lemma on the strong convexity of the total cost.
\begin{lemma}
    \label{lm:ac_convex}
    Suppose that $f$ is $m$-strongly convex and that the minimum and maximum eigenvalues of Q are $\frac{\alpha}{2}>0$ and $\frac{\beta}{2}$.  The total cost function $cost(x_{1:T})$ is $\alpha_1$-strongly convex, where $m < \alpha_1 \leq m+\frac{\beta}{T^2}$.
\end{lemma}

\begin{proof}
If $x_{1:T}^\pi$ is a series of actions given by a policy $\pi$ for the context sequence $y_{1:T}$, then the overall cost must satisfy
\begin{equation}
\begin{split}
    \label{eq:lemma_D1_main}
	cost(x_{1:T}) -  cost(x_{1:T}^\pi) &= \sum_{t=1}^T f(x_t) -f(x_t^\pi) + c(x_{t-1}, x_t)  -  c(x_{t-1}^\pi, x_t^\pi)\\
	&\geq \langle x_{1:T}-x_{1:T}^\pi, \nabla_{x_{1:T}^\pi} cost(x_{1:T}^\pi) \rangle + \frac{m}{2} \sum_{t=1}^T \| x_t - x_t^\pi \|^2 + \frac{1}{2} \Delta_X^T H \Delta_X
\end{split}
\end{equation}
where $H$ is the hessian matrix of the switching cost and $\Delta_X = [x_1 - x_1^\pi, x_2 - x_2^\pi, \cdots, x_T - x_T^\pi]^T$. By taking the second-order partial derivatives of $ \sum_{t=1}^T c(x_{t-1}, x_t)  -  c(x_{t-1}^\pi, x_t^\pi)$
with respect to $x_{1:T}^\pi$, we can find
\begin{equation}
H = 2\mx{2Q ,& -Q ,& 0 ,& \cdots , 0\\-Q ,&2Q ,&-Q ,& \cdots , 0\\ 0 ,& -Q ,& 2Q  ,& \cdots , 0\\ &\cdots&\\0 ,&0 ,&0 ,& \cdots, Q}.
\end{equation}
Next, we will prove that the Hessian matrix  $H$ is positive definite, and the minimum eigenvalue of $H$ is $\alpha_2$, where $0 < \alpha_2 \leq \frac{\alpha}{T^2}$. For a vector $X_1 = [x_1', x_2', \cdots, x_T']$, we have
\begin{equation}
\begin{split}
\frac{1}{2} X_1^T H X_1 = x_1'^T Q x_1' + \sum_{t=2}^T   (x_t' - x_{t-1}')^T Q (x_t' - x_{t-1}')
\end{split}
\end{equation}

Since the minimum eigenvalue of $Q$ is $\frac{\alpha}{2}>0$, it is clear that for any non-zero $X_1$, we have $\frac{1}{2} X_1^T H X_1 >0$, the Hessian matrix is positive definite. Thus, there exists
$\alpha_2 > 0$ such that for any non-zero $X_1$,  we have:
\begin{equation}
\label{eq:convex_nec}
\frac{1}{2} X_1^T H X_1 \geq \frac{\alpha_2}{2} \| X_1\|^2 > 0.
\end{equation}
If we construct $X_1$ as $x_1' = x_2' = x_3' = \cdots = x_T'$, then $\frac{1}{2} X_1^T H X_1 =  x_1'^T Q x_1' \leq  \frac{\beta}{2 T^2} \| X_1\|^2$. To ensure that the inequality \eqref{eq:convex_nec} always holds, we have to let $\alpha_2 \leq \frac{\beta}{T^2}$. To sum up, we have proved that $0 < \alpha_2 \leq \frac{\beta}{T^2}$. By substituting \eqref{eq:convex_nec} back to \eqref{eq:lemma_D1_main}, we have
\begin{equation}
\begin{split}
	cost(x_{1:T}) -  cost(x_{1:T}^\pi)
	&\geq \langle x_{1:T}-x_{1:T}^\pi, \nabla cost(x_{1:T}^\pi) \rangle + \frac{m + \alpha_2}{2} \sum_{t=1}^T \| x_t - x_t^\pi \|^2,
\end{split}
\end{equation}
Thus, $cost(x_{1:T})$ is $(m + \alpha_2)$-strongly convex,
where  $0 <\alpha_2 \leq \frac{\beta}{T^2}$. This completes the
proof by setting $\alpha_1=m+\alpha_2$.
\end{proof}

With Lemma~\ref{lm:ac_convex}, we now turn to the main proof of Theorem~\ref{thm:regret}.

\begin{proof}
Let $ x_{1:T}^L$ be the sequence of actions made the $L$-constrained offline oracle. Since $x_t$ is the solution to Line 5~of Algorithm~\ref{alg:RP-OBD}, we have
$$\nabla f(x_t, y_t) = 2\lambda_1 Q(x_{t-1} - x_t) + 2 \lambda_2 Q (v_t - x_t) + 2 \lambda_3 Q (\tilde{x}_t - x_t) $$
Since the hitting cost function $f(x_t, y_t)$ is $m$-strongly convex with respect to $x_t$, we have
\begin{equation}
    \label{eq:L_regret_gd}
    \begin{split}
        f(x_t, y_t) - f(x_t^L, y_t) \leq &\langle \nabla f(x_t, y_t), x_t - x_t^L \rangle - \frac{m}{2} \|x_t - x_t^L\|^2\\
        = &2\lambda_1 (x_t - x_t^L)^T Q(x_{t-1} - x_t) - \frac{m}{2} \|x_t - x_t^L\|^2  \\
        & +  2 \lambda_2 (x_t - x_t^L)^T Q (v_t - x_t)+ 2 \lambda_3 (x_t - x_t^L)^T Q (\tilde{x}_t - x_t)
    \end{split}
\end{equation}
In the last step of Eqn.~\eqref{eq:L_regret_gd}, we  have
\begin{equation}
    \label{eq:regret_convex}
    \begin{split}
        &(x_t - x_t^L)^T Q (\tilde{x}_t - x_t)\\
        =& (x_t - x_t^L)^T Q (\tilde{x}_t - x_t^L) + (x_t - x_t^L)^T Q (x^L_t - x_t)\\
        \leq& (x_t - x_t^L)^T Q (\tilde{x}_t - x_t^L)
        \leq  \frac{\beta}{2} \| x_t - x_t^L\| \| \tilde{x}_t - x_t^L\|\\
        \leq& \frac{\beta}{4} \big(\| x_t - x_t^L\|^2 + \| \tilde{x}_t - x_t^L\|^2 \big)
    \end{split}
\end{equation}
Here, given the online actions $x_{1:T}$, the regret is defined
as $Regret(x_{1:T}, x_{1:T}^L) = cost(x_{1:T}) - cost(x_{1:T}^L)$. From Lemma~\ref{lm:ac_convex}, we have
\begin{equation}
    Regret(x_{1:T}, x_{1:T}^L) \geq \langle \nabla cost(x_{1:T}^L), x_{1:T} - x_{1:T}^L \rangle + \frac{\alpha_1}{2} \| x_{1:T} - x_{1:T}^L \|^2.
\end{equation}
Since the action $x_{1:T}^L$ is
made by the  $L$-constrained optimal oracle, from the KKT condition,
we have $\nabla_{x_{1:T}^L} cost(x_{1:T}^L) + \mu \nabla_{x_{1:T}^L} \big( \sum_{t=1}^T c(x_t^L, x_{t-1}^L) \big) = 0$ and $\mu \big( \sum_{t=1}^T c(x_t^L, x_{t-1}^L) - L \big) = 0 $, where $\mu\geq0$ is the dual variable. 
If the total switching cost $\sum_{t=1}^T c(x_t^L, x_{t-1}^L) < L$, then $\mu = 0$ and $\nabla_{x_{1:T}^L} cost(x_{1:T}^L) = 0$. If the total switching cost $\sum_{t=1}^T c(x_t^L, x_{t-1}^L) = L$, then $\langle \nabla cost(x_{1:T}^L), x_{1:T} - x_{1:T}^L \rangle = - \mu \langle \nabla_{x_{1:T}^L} \big( \sum_{t=1}^T c(x_t^L, x_{t-1}^L) \big), x_{1:T} - x_{1:T}^L \rangle$. Since $x_{1:T}$ also satisfies the $L$-constraint and the switching cost is strongly convex, we have $\sum_{t=1}^T c(x_t, x_{t-1}) \geq \sum_{t=1}^T c(x_t^L, x_{t-1}^L) + \langle \nabla_{x_{1:T}^L} \left( \sum_{t=1}^T c(x_t^L, x_{t-1}^L) \right), x_{1:T} - x_{1:T}^L \rangle$. Given $\sum_{t=1}^T c(x_t^L, x_{t-1}^L) = L$ and
 $\sum_{t=1}^T c(x_t, x_{t-1}) \leq L$, we have $\langle \nabla_{x_{1:T}^L} \big( \sum_{t=1}^T c(x_t^L, x_{t-1}^L) \big), x_{1:T} - x_{1:T}^L \rangle \leq 0$. To sum up, we have $\langle \nabla cost(x_{1:T}^L), x_{1:T} - x_{1:T}^L \rangle \geq 0$.
Therefore, we have
\begin{equation}
    \label{eq:regret_dis}
  \sum_{t=1}^T \| x_{t} - x_{t}^L \|^2 =  \| x_{1:T} - x_{1:T}^L \|^2 \leq \frac{2}{\alpha_1} Regret(x_{1:T}, x_{1:T}^L).
\end{equation}
Substituting \eqref{eq:regret_dis} back to \eqref{eq:regret_convex}, if $\sum_{t=1}^T c(x_t, x_{t-1}) \leq L$,  we have
\begin{equation}
    \label{eq:bound_small_l3}
    \sum_{t=1}^T (x_t - x_t^L)^T Q (\tilde{x}_t - x_t) \leq \frac{ \beta}{2 \alpha} Regret(x_{1:T}, x_{1:T}^L)  +  \frac{\beta}{4} L_\rho.
\end{equation}
Based on Eqn.~\eqref{eq:BD_grad}, we have
\begin{equation}
\label{eq:bound_1}
 2(x_t - x_t^L)^T Q(x_{t-1} - x_t) \leq c(x_t^L, x_{t-1}) \leq\frac{\beta}{2}  \|x_t^L- x_{t-1} \|^2 \leq \frac{\beta \omega^2}{2},
\end{equation}
\begin{equation}
\label{eq:bound_2}
 2(x_t - x_t^L)^T Q(\tilde{x}_t - x_t) \leq c(x_t^L, \tilde{x}_t) \leq\frac{\beta}{2}  \|x_t^L- \tilde{x}_t \|^2 \leq \frac{\beta \omega^2}{2}.
\end{equation}
Substituting \eqref{eq:BD_grad} and \eqref{eq:bound_2} back to \eqref{eq:L_regret_gd}, we introduce a new auxiliary parameter $q$ as follows:
\begin{equation}
    \label{eq:L_regret_part}
    \begin{split}
        &f(x_t, y_t) - f(x_t^L, y_t) \\
        \leq &\lambda_1 \left( c(x_t^L, x_{t-1}) - c(x_t^L, x_t) - c(x_t, x_{t-1}) \right) - \frac{m}{2} \|x_t - x_t^L \|^2 + \lambda_2 \frac{\beta \omega^2}{2} + 2\lambda_3 (x_t - x_t^L)^T Q (\tilde{x}_t - x_t)  \\
        = & (\lambda_1 + q) \left( c(x_t^L, x_{t-1}) - c(x_t^L, x_t) \right)  - \lambda_1  c(x_t, x_{t-1})   + \lambda_2 \frac{\beta \omega^2}{2} + 2 \lambda_3 (x_t - x_t^L)^T Q (\tilde{x}_t - x_t) - \\
        &\left(q \cdot c(x_t^L, x_{t-1})  - q \cdot c(x_t^L, x_{t})  + \frac{m}{2} \|x_t - x_t^L \|^2 \right).
    \end{split}
\end{equation}
In the last step of \eqref{eq:L_regret_part}, we have
\begin{equation}
    \begin{split}
        &q \cdot c(x_t^L, x_{t-1})  - q \cdot c(x_t^L, x_{t})  + \frac{m}{2} \|x_t - x_t^L \|^2  \\
        =& q \cdot \big( c(x_t, x_{t-1}) - 2(x_{t-1} - x_{t})^T Q (x_t - x_t^L) \big) + \frac{m}{2} \|x_t - x_t^L \|^2 \\
         \geq& q\cdot c(x_t, x_{t-1}) - 2 q \| Q (x_{t-1} - x_{t}) \| \|x_t - x_t^L \| + \frac{m}{2} \|x_t - x_t^L \|^2 \\
         \geq & q\cdot c(x_t, x_{t-1}) - \big( \frac{2q^2}{m} \| Q (x_{t-1} - x_{t}) \|^2 + \frac{m}{2} \|x_t - x_t^L \|^2 \big) + \frac{m}{2} \|x_t - x_t^L \|^2  \\
         =& q\cdot c(x_t, x_{t-1}) - \frac{2q^2}{m}  (x_{t-1} - x_{t})^T Q^2 (x_{t-1} - x_{t})\\
         \geq&  (q - \frac{\beta q^2}{m}) c(x_t, x_{t-1}).
    \end{split}
\end{equation}
 We set $q = \frac{m}{2\beta}$, substitute it back to  \eqref{eq:L_regret_part},
 and get
\begin{equation}
    \begin{split}
        &f(x_t, y_t) - f(x_t^L, y_t) \\
        \leq& (\lambda_1 + \frac{m}{2\beta})\big(c(x_t^L, x_{t-1}) - c(x_t^L, x_t) \big) - (\lambda_1 + \frac{m}{4\beta})c(x_t, x_{t-1}) +\lambda_2 \frac{\beta \omega^2}{2} + 2\lambda_3 (x_t - x_t^L)^T Q (\tilde{x}_t - x_t) \\
        \leq &(\lambda_1 + \frac{m}{2\beta})\big(c(x_t^L, x_{t-1}) - c(x_t^L, x_t) \big) - c(x_t, x_{t-1}) + \lambda_2 \frac{\beta \omega^2}{2} + 2\lambda_3 (x_t - x_t^L)^T Q (\tilde{x}_t - x_t),
    \end{split}
\end{equation}
where the second inequality comes from the assumption $\lambda_1 + \frac{m}{4\beta} \geq 1$.

Summing up  $c(x_t, x_{t-1})$ from $1$ to $T$, we can get
\begin{equation}
    \begin{split}
        &\sum_{i = 1}^T c(x_t^L, x_{t-1}) - c(x_t^L, x_t)\\
        =& \sum_{i = 1}^T x_{t-1}^T Q x_{t-1} - x_t^T Q x_t  +  (x_t - x_{t-1})^T Q x_t^L
        = x_{0}^T Q x_{0} - x_{T}^T Q x_{T} +  \sum_{t = 1}^T  x_t^T Q x_t^L - \sum_{t = 0}^{T-1}  x_{t}^T Q x_{t+1}^L \\
        =& x_{0}^T Q x_{0} - x_{0}^T Q x_{1}^L + \sum_{t = 1}^{T-1}  x_t^T Q x_t^L - \sum_{t = 1}^{T-1}  x_{t}^T Q x_{t+1}^L
        \leq \sum_{t = 1}^{T-1}  x_t^T Q (x_t^L - x_{t+1}^L) \leq \sum_{t = 1}^{T-1}  \| x_t^T Q \| \cdot \|x_t^L - x_{t+1}^L\| \\
        \leq& G \sum_{t = 1}^{T} \|x_t^L - x_{t+1}^L\|.
    \end{split}
\end{equation}
By the generalized mean inequality, we can have $$\sum_{t = 1}^{T} \|x_t^L - x_{t+1}^L\| \leq \sqrt{T(\sum_{t = 1}^{T} \|x_t^L - x_{t+1}^L\|^2)} \leq \sqrt{\frac{2T}{\alpha}\left(\sum_{t = 1}^{T} (x_t^L - x_{t+1}^L)^T Q (x_t^L - x_{t+1}^L) \right)} \leq \sqrt{\frac{2TL}{\alpha}}.$$
Substituting this result into the regret, we have
\begin{equation}
    \begin{split}
    \label{eq:final_regret}
        &Regret(x_{1:T}, x_{1:T}^L) = cost(x_{1:T}) - cost(x_{1:T}^L)\\
        =&\sum_{t=1}^T \big( f(x_t, y_t) + c(x_t, x_{t-1}) \big) - \sum_{t=1}^T \big( f(x_t^L, y_t) + c(x_t^L, x_{t-1}^L) \big)\\
        \leq  &\sum_{t=1}^T \big( f(x_t, y_t) - f(x_t^L, y_t) + c(x_t, x_{t-1}) \big) \\
        \leq& (\lambda_1 + \frac{m}{2\beta} )G \sqrt{\frac{2TL}{\alpha}} + \lambda_2 \frac{\beta T \omega^2}{2} +  2\lambda_3 \sum_{t=1}^T (x_t - x_t^L)^T Q (\tilde{x}_t - x_t).
    \end{split}
\end{equation}
Therefore, if  $\lambda_3< \frac{\alpha}{\beta}$ and
$\sum_{t=1}^T c(x_t, x_{t-1}) \leq L$ are satisfied, we can substitute \eqref{eq:bound_small_l3} back to \eqref{eq:final_regret}, move $Regret(x_{1:T}, x_{1:T}^L)$ to left side, and obtain
$$Regret(x_{1:T}, x_{1:T}^L) \leq \frac{\alpha}{\alpha - \lambda_3 \beta} \big( (\lambda_1 + \frac{m}{2\beta} )G \sqrt{\frac{2TL}{\alpha}} + \lambda_2 \frac{\beta T \omega^2}{2} + \frac{\lambda_3 \beta}{2} L_\rho \big).$$
Otherwise,
we  substitute \eqref{eq:bound_2} back to \eqref{eq:final_regret} and obtain
$$Regret(x_{1:T}, x_{1:T}^L) \leq  (\lambda_1 + \frac{m}{2\beta} )G \sqrt{\frac{2TL}{\alpha}} + (\lambda_2 +\lambda_3)  \frac{\beta T \omega^2}{2}.$$
\end{proof} 

\section{Proof of Theorem~\ref{thm:avgcostclb}}\label{sec:proofavgcost}
\begin{proof}
We denote the empirical distribution on the dataset $\mathcal{D}$ as $\mathbb{P}_{\mathcal{D}}$. Thus, $D(\mathbb{P}_{\mathcal{D}},\mathbb{P})$ is the Wasserstein distance between the training empirical distribution and testing distribution. To be more general, we assume that  $\mathcal{D}$ is sampled from a distribution $\mathbb{P}'$ which may be different from $\mathbb{P}$. By the properties of Wasserstein distance \cite{Wasserstein_kuhn2019wasserstein}, we have with probability at least $1-\delta, \delta\in(0,1)$,
\begin{equation}\label{eqn:empiricaldisdisc}
\begin{split}
	D\left(\mathbb{P}, \mathbb{P}_{\mathcal{D}}\right) \leq D\left(\mathbb{P'}, \mathbb{P}_{\mathcal{D}}\right)+D\left(\mathbb{P}, \mathbb{P}'\right) \leq\sqrt{\frac{\log(b_1/\delta)}{b_2|\mathcal{D}|}}+D\left(\mathbb{P}, \mathbb{P}'\right),
	\end{split}
	\end{equation}
where $b_1$ and $b_2$ are two positive constants. We can find that the distributional discrepancy $D\left(\mathbb{P}', \mathbb{P}_{\mathcal{D}}\right) $ becomes smaller if the empirical distribution has more samples ($|\mathcal{D}|$ is larger).

Denote $L_{R_{\lambda},\mu}(h_{W},\bm{s})=\mu l(h_W,\bm{s}) +(1-\mu)\mathrm{cost}(R_{\lambda}\circ h_W,\bm{s})$ as the loss function of the training objective of \ouralgpro.
  By the generalization bound based on Wasserstein measure \cite{Wasserstein_kuhn2019wasserstein}, we have with probability at least $1-\delta, \delta\in(0,1)$,
  	\begin{equation}\label{eqn:proofgeneralization}
		\begin{split}
		\mathbb{E}\left[ L_{R_{\lambda},\mu}(h_{\hat{W}},\bm{s})\right]&\leq L_{\mathcal{D},R_{\lambda},\mu}(h_{\hat{W}})+O\left( D(\mathbb{P},\mathbb{P}_{\mathcal{D}})\right)\\
&=L_{\mathcal{D},R_{\lambda},\mu}(h_{\hat{W}^*})+\mathcal{E}_{\mathcal{D}} +O\left( D(\mathbb{P},\mathbb{P}_{\mathcal{D}})\right) \\
&\leq L_{\mathcal{D},R_{\lambda},\mu}(h_{W^*})+\mathcal{E}_{\mathcal{D}} +O\left( D(\mathbb{P},\mathbb{P}_{\mathcal{D}})\right) \\
&\leq \mathbb{E}\left[L_{R_{\lambda},\mu}(h_{W^*},\bm{s})\right]+\mathcal{E}_{\mathcal{D}}+O\left( D(\mathbb{P},\mathbb{P}_{\mathcal{D}})\right),
\end{split}
\end{equation}
where the equality comes from the definition of $\mathcal{E}_{\mathcal{D}}$ in Eqn.~\eqref{eqn:training_error}, the second inequality holds because $h_{\hat{W}^*}$ minimizes the empirical loss $L_{\mathcal{D},R_{\lambda},\mu}(h_{\hat{W}})$, and the last inequality holds by using the generalization bound again.

Since $\mathbb{E}\left[ L_{R_{\lambda},\mu}(h_{W},\bm{s})\right]=\mu\mathbb{E}\left[l(h_{W},\bm{s})\right]+(1-\mu)\mathbb{E}\left[\mathrm{cost}(R_{\lambda}\circ h_{W},\bm{s})\right]$, by Eqn.~\eqref{eqn:proofgeneralization}, we have with probability $1-\delta,\delta\in(0,1)$,
\begin{equation}\label{eqn:proofgeneralization2}
  \begin{split}
      \mathrm{AVG}&(R_{\lambda}\circ h_{\hat{W}})=\mathbb{E}\left[\mathrm{cost}(R_{\lambda}\circ h_{\hat{W}},\bm{s})\right]\leq \frac{1}{1-\mu}\mathbb{E}\left[ L_{R_{\lambda},\mu}(h_{\hat{W}},\bm{s})\right]\\
      &\leq \frac{1}{1-\mu}\mathbb{E}\left[L_{R_{\lambda},\mu}(h_{W^*},\bm{s})\right]+\frac{1}{1-\mu}\mathcal{E}_{\mathcal{D}} +O\left( D(\mathbb{P},\mathbb{P}_{\mathcal{D}})\right)\\
      &=\mathrm{AVG}(R_{\lambda}\circ h_{W^*})+\frac{\mu}{1-\mu}\mathbb{E}\left[l(h_{W^*},\bm{s})\right]+\frac{1}{1-\mu}\mathcal{E}_{\mathcal{D}} +O\left( D(\mathbb{P},\mathbb{P}_{\mathcal{D}})\right).
  \end{split}
\end{equation}

The proof is completed by substituting \eqref{eqn:empiricaldisdisc} into \eqref{eqn:proofgeneralization2}.
	\end{proof}

\section{Proof of Theorem \ref{thm:tail_ratio}}\label{sec:proof_tailratio}
\begin{proof}
Since each testing sequence $\bm{s}$ is sampled from a distribution $\mathbb{P}$, we use McDiarmid's inequality for dependent random variables \cite{Markov_concentration_paulin2015concentration} to bound $\mathrm{loss}(h_{W},\bm{s})$.

Denote the prediction error for sequence $\bm{s}$ at step $t,t\in[T]$ as $\mathrm{error}_{\bm{s},t}=\frac{\|\tilde{x}_t-x_t^*\|^2}{\mathrm{cost}(\pi^*,\bm{s})}$ and $\mathrm{error}_{\bm{s}}$ as the sequence of prediction error. The loss function $l(h_{W},\bm{s})$ for $\bm{s}$ can also be denoted as $l(\mathrm{error}_{\bm{s}})=\mathrm{relu}(\sum_{t=1}^T\mathrm{error}_{\bm{s},t}-\bar{\rho})$.

By the assumptions that $\inf_{\bm{s}\in\mathcal{S}}\mathrm{cost}(\pi^*,\bm{s})=\nu$ and $\sup_{x,y\in\mathcal{X}}\|x-y\|=\omega$, we have $\mathrm{error}_{\bm{s},t}\leq \frac{\omega^2}{\nu}$. Thus, if one entry of $\mathrm{error}_{\bm{s}}$ is changed, $l(\mathrm{error}_{\bm{s}})$ increases or decreases at most $\frac{2\omega^2}{\nu}$, and so for two sequences of prediction errors $\mathrm{error}_{\bm{s}}$ and $\mathrm{error}_{\bm{s}'}$, we have
\begin{equation}\label{eqn:prooftailratio1}
l(\mathrm{error}_{\bm{s}})-l(\mathrm{error}_{\bm{s}'})\leq \sum_{t=1}^T\frac{2\omega^2}{\nu}\mathds{1}(\mathrm{error}_{\bm{s},t}\neq \mathrm{error}_{\bm{s}',t}),
\end{equation} where $\mathds{1}$ is the indicator function.

With the inequality \eqref{eqn:prooftailratio1}, we can use Theorem 2.1 in \cite{Markov_concentration_paulin2015concentration} for the random sequence $\mathrm{error}_{\bm{s}}$. We have for any $\epsilon>0$
\begin{equation}
\mathbb{P}\left( \left|l(\mathrm{error}_{\bm{s}})-\mathbb{E}\left[l(\mathrm{error}_{\bm{s}})\right]\right|\geq \epsilon\right)\leq 2\exp(\frac{-2\epsilon^2}{b_3\|\Gamma\|^2T(\frac{2\omega^2}{\nu})^2}),
\end{equation}
where $\Gamma$ is the mixing matrix of the Marton coupling for a partition of the random sequence $\mathrm{error}_{\bm{s}}$, as is defined in Definition 2.1 in \cite{Markov_concentration_paulin2015concentration} and $b_3$ is a constant related to the size of the partition (Definition 2.3 in \cite{Markov_concentration_paulin2015concentration}).

Let $2\exp(\frac{-2\epsilon^2}{b_3\|\Gamma\|^2T(\frac{2\omega^2}{\nu})^2})=\delta'\in(0,1)$ and absorb the constants by $O$ notation. Then, we have with probability at least $1-\delta',\delta'\in(0,1)$,
\begin{equation}\label{eqn:prooftailcost1}
    l(h_{W},\bm{s})\leq \mathbb{E}\left[ l(h_{W},\bm{s})\right]+O\left(\frac{\omega^2 \sqrt{T}}{\nu}\|\Gamma\|\sqrt{\frac{1}{2}\log(\frac{2}{\delta'})}\right).
\end{equation}

Since $\mathbb{E}\left[ L_{R_{\lambda},\mu}(h_{W},\bm{s})\right]=\mu\mathbb{E}\left[l(h_{W},\bm{s})\right]+(1-\mu)\mathbb{E}\left[\mathrm{cost}(R_{\lambda}\circ h_{W},\bm{s})\right]$, by Eqn.~\eqref{eqn:proofgeneralization} and Eqn.~\eqref{eqn:empiricaldisdisc}, we have with probability $1-\delta,\delta\in(0,1)$,
\begin{equation}\label{eqn:prooftailcost}
\begin{split}
    \mathbb{E}&\left[l(h_{\hat{W}},\bm{s})\right]\leq \frac{1}{\mu}\mathbb{E}\left[ L_{R_{\lambda},\mu}(h_{\hat{W}},\bm{s})\right]\\
    &\leq \frac{1}{\mu}\mathbb{E}\left[L_{R_{\lambda},\mu}(h_{W^*},\bm{s})\right]+\frac{1}{\mu}\mathcal{E}_{\mathcal{D}} +O\left( D(\mathbb{P},\mathbb{P}_{\mathcal{D}})\right)\\
    &=\mathbb{E}\left[l(h_{W^*},\bm{s})\right]+\frac{1-\mu}{\mu}\mathbb{E}\left[\mathrm{cost}(R_{\lambda}\circ h_{W^*},\bm{s})\right]+\frac{1}{\mu}\mathcal{E}_{\mathcal{D}} +O\left( \sqrt{\frac{\log(1/\delta)}{|\mathcal{D}|}}+D\left(\mathbb{P}, \mathbb{P}'\right)\right).
    \end{split}
\end{equation}
Substituting \eqref{eqn:prooftailcost} into \eqref{eqn:prooftailcost1}, we have with probabiltiy at least $1-\delta-\delta', \delta>0,\delta'>0,(\delta+\delta')\in(0,1)$,
\begin{equation}
\begin{split}
    l(h_{\hat{W}},\bm{s})\leq& \mathbb{E}\left[l(h_{W^*},\bm{s})\right]+\frac{1-\mu}{\mu}\mathbb{E}\left[\mathrm{cost}(R_{\lambda}\circ h_{W^*},\bm{s})\right]+\frac{1}{\mu}\mathcal{E}_{\mathcal{D}} \\
    &+O\left( \sqrt{\frac{\log(1/\delta)}{|\mathcal{D}|}}+D\left(\mathbb{P}, \mathbb{P}'\right)\right)+O\left(\frac{\omega^2 \sqrt{T}}{\nu}\|\Gamma\|\sqrt{\frac{1}{2}\log(\frac{2}{\delta'})}\right).
    \end{split}
\end{equation}
By the definition of $l(h_{\hat{W}},\bm{s})=\mathrm{relu}\left(\frac{\sum_{t=1}^T\|\tilde{x}_t-x_t^*\|^2}{\mathrm{cost}(\pi^*,\bm{s})}-\bar{\rho}\right)$ and the choices of $\delta$ and $\delta'$, we have with probabiltiy at least $1-\delta, \delta\in(0,1)$,
\begin{equation}
\begin{split}
    \frac{\sum_{t=1}^T\|\tilde{x}_t-x_t^*\|^2}{\mathrm{cost}(\pi^*,\bm{s})}\leq \bar{\rho}&+\mathbb{E}\left[l(h_{W^*},\bm{s})\right]+\frac{1-\mu}{\mu}\mathbb{E}\left[\mathrm{cost}(R_{\lambda}\circ h_{W^*},\bm{s})\right]+\frac{1}{\mu}\mathcal{E}_{\mathcal{D}} \\
    &+O\left( \sqrt{\frac{\log(2/\delta)}{|\mathcal{D}|}}+D\left(\mathbb{P}, \mathbb{P}'\right)\right)+O\left(\frac{\omega^2 \sqrt{T}}{\nu}\|\Gamma\|\sqrt{\frac{1}{2}\log(\frac{4}{\delta})}\right)=\rho_{\mathrm{tail}}.
    \end{split}
\end{equation}
Thus, by Theorem~\ref{thm:cr_robd}, we have with probability at least $1-\delta, \delta\in(0,1)$,
$$\frac{\mathrm{cost}(R_{\lambda}\circ h_{\hat{W}},\bm{s})}{\mathrm{cost}(\pi^*,\bm{s})}\leq 1 + \frac{1}{2}
\left[\sqrt{(1+\frac{ \beta }{m}\theta)^2 + \frac{4 \beta^2}{m \alpha}} - \left(1 + \frac{\beta}{m}\theta)\right)\right] + \frac{\beta}{2} \theta\cdot\rho_{\mathrm{tail}}.$$
\end{proof} 
\end{document}